%% file: egpaper_final.tex
\documentclass[10pt,twocolumn,letterpaper]{article}

\usepackage{iccv}
\usepackage{times}
\usepackage{epsfig}
\usepackage{graphicx}
\usepackage{amsmath}
\usepackage{amssymb}
\usepackage[pagebackref=true,breaklinks=true,letterpaper=true,colorlinks,bookmarks=false]{hyperref}
\iccvfinalcopy 

\usepackage{braket} 
\usepackage{enumitem} 
\usepackage{wrapfig}
\usepackage{color}
\definecolor{darkorange}{rgb}{1.0, 0.55, 0.0} 
\usepackage{caption}
\usepackage[labelsep=period, labelfont=bf, font=small]{caption} 
\usepackage{subcaption}
\usepackage{comment}
\usepackage{physics}
\usepackage{float} 
\usepackage{array}
\usepackage{dsfont}
\usepackage{overpic}
\makeatletter
\@namedef{ver@everyshi.sty}{}
\makeatother
\usepackage{pgfplots}
\usepackage{pgf,tikz,tikzscale}
\usetikzlibrary{calc}
\usepackage[ruled, linesnumbered]{algorithm2e}
\usepackage{amsthm} 

\DeclareMathOperator*{\argmin}{arg\,min}

\newtheorem{lemma}{Lemma}[section]

\begin{document} 

\title{Q-Match: Iterative Shape Matching via Quantum Annealing\vspace{-8pt}} 

\author{
\hspace{22pt}
\begin{tabular}{ccc}
    Marcel Seelbach Benkner$^{1}$ $\quad$ & Zorah L\"{a}hner$^{1}$ $\quad$ & Vladislav Golyanik$^{2}$ $\quad$ \vspace{5pt}\\
    Christof Wunderlich$^{1}$ $\quad$ & Christian Theobalt$^{2}$ $\quad$ & Michael Moeller$^{1}$ $\quad$  
    \end{tabular} 
    \vspace{11pt}\\
    $^{1}$University of Siegen $\quad\quad\quad$ $^{2}$MPI for Informatics, SIC 
} 

\makeatletter
\let\@oldmaketitle\@maketitle%
\renewcommand{\@maketitle}{\@oldmaketitle%
  \myfigure{}\bigskip}%
\makeatother

\newcommand\myfigure{%
  \centering
    (i)
    \begin{overpic}[width=.1\linewidth]{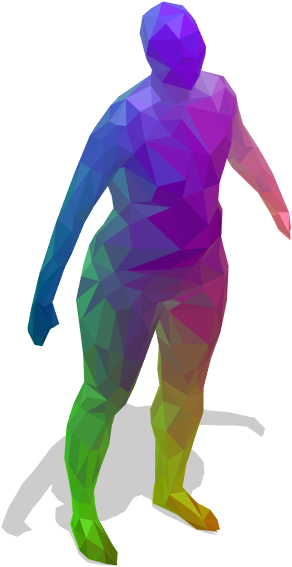} 
        \put(40,110){Q-Match}
    \end{overpic}
    \begin{overpic}[width=.1\linewidth]{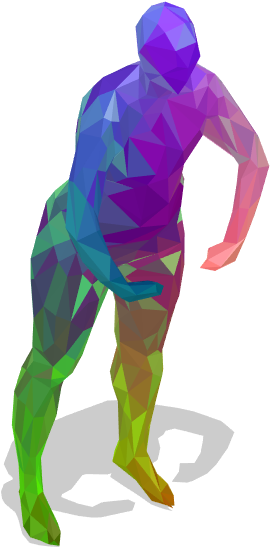} 
        \put(90,105){QGM \cite{seelbach20quantum}}
    \end{overpic}
    \includegraphics[width=.1\linewidth]{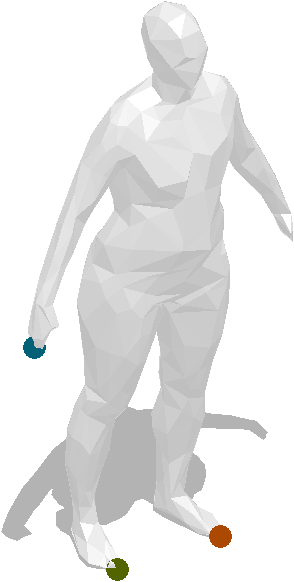} 
    \includegraphics[width=.1\linewidth]{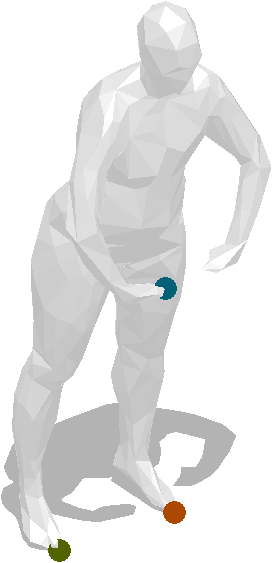} 
    (ii)
    \includegraphics{FIGURES/cyclic-alpha.tikz}
    \captionof{figure}{(i) Example correspondences obtained by two quantum approaches, the proposed \textit{Q-Match} (left, $502$ points) and QGM \cite{seelbach20quantum} (right, three points). 
    Vertices with the same color are mapped to each other. Our iterative method operates directly in the space of permutations and thus can deal with many more vertices compared to QGM. 
    (ii) We represent permutations through a collection of cycles, which allows us to parametrize them through binary variables without the need to enforce any kind of constraints (upper right). 
    \textit{Q-Match} makes the decision whether each cycle $C_i$ should be applied ($\alpha_i=1$) or not ($\alpha_i=0$) to progress towards a lower energy solution.
    }
    \label{fig:teaser}
  }

\maketitle

\begin{abstract} 
Finding shape correspondences can be formulated as an $\mathcal{NP}$-hard quadratic assignment problem (QAP) 
that 
becomes infeasible for shapes with high sampling density. 
A promising research direction is to tackle such quadratic optimization problems over binary variables with \textit{quantum annealing}, which allows for some problems a more efficient search in the solution space. 
Unfortunately, enforcing the linear equality constraints in QAPs via a penalty significantly limits the success probability of such methods on currently available quantum hardware. 
To address this limitation, this paper proposes Q-Match, \textit{i.e.,} a new iterative quantum method for QAPs inspired by the $\alpha$-expansion algorithm, which allows solving problems of an order of magnitude larger than current quantum methods. 
It implicitly enforces the QAP constraints by updating the current estimates in a cyclic fashion. 
Further, Q-Match can be applied 
iteratively, on a subset of well-chosen correspondences, allowing us to scale to real-world problems. 
Using the latest quantum annealer, the D-Wave Advantage, we evaluate the proposed method on a subset of QAPLIB as well as on isometric shape matching problems from the FAUST dataset. 
\vspace{-37pt} 
\end{abstract} 

\section{Introduction}\label{sec:introduction} 

Shape correspondence is at the heart of many computer vision and graphics applications, as knowing the relationship between points allows transferring  information to new shapes. 
While this problem has been around for a long time \cite{loncaric98survey}, the setting with non-rigidly deformed shapes is still challenging \cite{dyke19shrec}. 
One reason is that generic 
matching of $n$ points in two shapes can be formulated as an $\mathcal{NP}$-hard Quadratic Assignment Problem (QAP) 
\begin{equation} 
\min_{X\in \mathbb{P}_n} \; E(X ):= \mathbf{x} ^{\text{T}}W \mathbf{x}    , \label{eq:DASProblem}
\end{equation}
where $\mathbf{x} = \operatorname{vec}(X)$ is the $n^2$-dimensional vector corresponding to a matrix in the set of permutations $\mathbb{P} \subset \{ 0, 1 \}^{n\times n}$ and  
$W \in \mathbb{R}^{n^2 \times n^2}$ is an energy matrix describing how well certain pairwise properties are preserved between pairs of matches. 
For some choices of $W$ and assumptions on the input shapes, \eqref{eq:DASProblem} can become feasible by adding prior knowledge and exploiting geometric structures. 
For example, rigid shapes can be aligned through rotation and translation only --- a problem with only six degrees of freedom. 
For general cases, QAPs remain challenging nonetheless. 

In this work, we propose a new method for shape matching that  employs \textit{quantum annealing}, finds high quality solutions of QAPs with high probability, and is suitable for real-world problems. 
In recent years, quantum computers have made significant  progress and machines such as the IBM Q System One (circuit model) and D-Wave 
Advantage System $1.1$ (adiabatic, quantum annealer) became available for research purposes.
The circuit-based and adiabatic 
models 
are polynomially equivalent to each other 
in theory 
\cite{aharonov2008adiabatic}. However, the current implementations differ substantially and have different challenges to overcome. 
One advantage of adiabatic quantum computing %
is that it is less sensitive to noise and decoherence  \cite{childs2001robustness}. 
Unfortunately, adiabatic quantum computers (AQCer) can a-priori only solve unconstrained problems and thus cannot directly solve QAPs \eqref{eq:DASProblem} which are constrained to permutation matrices. %
To overcome this limitation, a recent method \cite{seelbach20quantum}
adds a penalty term enforcing the solution to be a permutation. 
However, this reduces the solvable problem sizes to very small $n \leq 4$ and,
in practice, the probability of finding the globally optimal solution drops to random guessing. 
In contrast, we propose a new formulation for AQCing that is guaranteed to result in permutation matrices 
in problem sizes in the order of magnitude of 
available (fully connected) qubits. See  Fig.~\ref{fig:teaser} for the method overview and Table  \ref{tab:notions} for a list of acronyms we use. 
Our method is iterative, and we solve a series of quadratic unconstrained binary optimization (QUBO) problems on D-Wave which were prepared on a classical CPU. 
In summary, the \textbf{contributions} of this paper are as follows: 
\begin{itemize} [leftmargin=*]\itemsep0em
    \item Q-Match, \textit{i.e.,} a new scalable quantum approach for searching correspondences between two shapes related by a non-rigid transformation (isometry or near-isometry). 
    \item An iterative problem formulation through a variant of $\alpha$-expansion which avoids explicit linear constraints for permutations and can be applied to large problem instances. 
\end{itemize} 

We evaluate our method on a real AQCer, D-Wave Advantage system $1.1$, and obtain improved results compared to the previous quantum method \cite{seelbach20quantum}. 
Moreover, Q-Match \textit{performs on par with the classical matching approach    \cite{ovsjanikov2012functional}}. 
This paper does not assume prior knowledge of quantum computing from readers. 
Secs.~\ref{ssec:quantum_annealing} and \ref{ssec:programming_AQC} review %
preliminaries required to understand, implement and apply the proposed Q-Match approach. 
The source code is available at  \url{https://4dqv.mpi-inf.mpg.de/QMATCH/}. 

\begin{table} 
\centering 
\small
\begin{tabular}{|c|p{160pt}|} 
\hline 
\textbf{acronym}  & $\quad\quad\quad\quad\quad\quad\;\;$\textbf{meaning} \\ 
\hline 
AQCer/ing           & adiabatic quantum computer/ing \\
QPU           & quantum processing unit \\
QAP           & quadratic assignment problem \\
QUBO          & quadratic unconstrained binary optimization \\ 
\hline 
\end{tabular} 
\caption{Commonly used acronyms and their meaning.} 
\label{tab:notions} 
\end{table}

\section{Related Work} \label{sec:relatedwork} 

This section reviews prior work on the intersection of quantum computing, computer vision and shape matching. 
For the foundations of AQCing, 
see Sec.~\ref{ssec:quantum_annealing} and  \cite{KadowakiNishimori1998, Farhi2001}. 

\noindent\textbf{AQCing for Computer Vision:} 
Fueled by the progress in quantum hardware accessible for a broad research community \cite{DWave_Leap, IBMQ}, a growing interest is developing to apply quantum computing or its concepts to computer vision. 
Ways for representing, retrieving and  processing images on a quantum computer have been  extensively investigated in the theory literature  \cite{VenegasAndraca2003, CaraimanManta2012, Sun2013,  Yan2016}. 
Methods for image recognition and classification are among the first low-level techniques evaluated on a real AQCer \cite{Neven2012, Boyda2017, Nguyen2018,  Cavallaro2020, LiGhosh2020}. 
O'Malley \textit{et al.}~\cite{OMalley2018} learn facial features and reproduce image  collections of human faces with the help of AQCing. 
Cavallaro \textit{et al.}~\cite{Cavallaro2020} classify multi-spectral images with an  ensemble of quantum support vector machines \cite{Willsch2020}. 
To account for the limited connectivity of the physical qubit graph of D-Wave 2000Q, they split the training set into multiple disjoint subsets and train the classifier on each of them  independently. %
Li and Ghosh \cite{LiGhosh2020} eliminate false positives in  multi-object detection on D-Wave 2X using the QUBO formulation from \cite{RujikietgumjornCollins2013}. 
It provides state-of-the-art accuracy and the implementation with quantum annealing is faster than greedy and tabu search. 
Solving image matching problems with a quantum annealer was first theoretically studied in \cite{neven2008image}. 
Practically applicable quantum algorithms for the absolute orientation problem and  point set alignment have been introduced in \cite{golyanik2020quantum}. 
The work shows that representing rotation matrices using a linear basis allows mapping both problems to QUBOs. 
A new point set alignment method for gate model quantum computers has  been recently derived in \cite{Noormandipour2021arXiv}, 
inspired by the earlier kernel correlation approach  \cite{TsinKanade2004}. 
QGM \cite{seelbach20quantum} is the first %
implementation of graph matching for small problem instances as a single QUBO with a penalty approach, 
whereas QuantumSync \cite{QuantumSync2021} is the first method for  permutation synchronization developed for and tested on an AQCer (Advantage system $1.1$). 
In \cite{seelbach20quantum}, the formulation of permutation matrix constraints with a linear term leads to low probabilities of measuring globally optimal solutions in a single annealing cycle on a modern AQCer. 
In contrast to \cite{seelbach20quantum, QuantumSync2021}, we avoid explicit constraints ensuring valid permutations and use a series of QUBOs leading %
successive improvements in the energy.
Consequently, we can align significantly larger graphs and shapes. 

\noindent\textbf{Shape Correspondence:} 
Finding point-wise matches between meshes is an actively studied problem in vision and graphics. 
One common way to establish matches between shapes related by isometry is to compare shape functions or signatures. 
Multiple methods of this category analyze spectra of the Laplace-Beltrami operator on the shape surfaces which are invariant under isometric deformations \cite{sun09hks, Aubry2011, ovsjanikov2012functional, Salti2014, Halimi2019}. 
Functional maps \cite{ovsjanikov2012functional, Halimi2019} interpret the problem as the alignment of functions in the pre-defined basis 
(\textit{e.g.,}  
eigenfunctions of the Laplace-Beltrami operator). 
Post-processing is required to extract point-wise matches. 
Among the shape descriptors, wave kernel signature (WKS) \cite{Aubry2011} is inspired by quantum physics. 
It relies on solving the Schrödinger equation for the dynamics (dissipation) of quantum-mechanical particles on the shape surface. 
WKS can resolve fine  details 
and is robust to moderate non-isometric  perturbations. 
Recently, convolution operators were generalized to non-Euclidean structures, and the availability of large-scale shape collections enabled supervised learning of dense shape correspondences \cite{Masci2015, Monti2017, Litany2017}. 

The point-wise correspondence search between two shapes can be formulated as a linear (LAP) or quadratic assignment problem (QAP) over the space of permutations, with the matching costs computed relying on feature descriptors \cite{sun09hks, Aubry2011, Salti2014}. 
The solution space of both problems is exponential in the number of points. 
For LAP, there exists a fast auction algorithm with polynomial runtime \cite{bertsekas98vt}. 
Per construction, LAP formulations for matching are sensitive to surface noise and do not explicitly consider spatial point relations, which often leads to local minima (and, consequently, inaccurate solutions). 
QAPs \cite{koopmans57qap, lawler1963qap} add quadratic costs for matching point pairs and regard point neighborhoods; the solutions are spatially smoother. 
The downside is their $\mathcal{NP}$-hardness, which makes finding global optima for large inputs unfeasible. 
Multiple policies to solve QAPs efficiently have been proposed in the literature, such as branch-and-bound \cite{Roucairol1987, Hahn1998}, spectral analysis \cite{leordeanu05spectral}, alternating direction method of multipliers  \cite{lehuu2017adgm}, entropic regularization of the energy landscape  \cite{Solomon2016} and simulated annealing \cite{holzschuh20simanneal}. 
Convex relaxations are among the most thoroughly investigated policies for QAPs \cite{Zhao1997, Anstreicher2001, Povh2009, Fogel2013, Kezurer2015, Bernard2018}. 
These methods either have prohibitive worst-case runtime complexity (branch-and-bound), rely on heuristic algorithmic choices (\textit{e.g.,} entropic regularization and simulated annealing), or do not guarantee global optima. 
In contrast, we address QAPs with a new AQCing metaheuristic to find high quality solutions with a high probability.

Holzschuh \textit{et  al.}~\cite{holzschuh20simanneal} proposed a non-quantum algorithm closely related to ours. 
It can be seen a special case of our setup that only applies one $2$-cycle per iteration. 
This means a considerably smaller step size compared to Q-Match and, hence, more iterations until convergence and higher sensitivity to local optima. 
Several heuristics to gain robustness and speed are further used in \cite{holzschuh20simanneal}. 
First, the algorithm is based on simulated annealing and accepts worse states with a certain probability to step over local optima. 
In contrast, our Q-Match explores larger solution spaces simultaneously 
and converges in fewer iterations. 
Second, their hierarchical optimization scheme adds new matches based on geodesic embeddings whereas we can refine any given initialization. 

\section{Background}\label{sec:background}

\subsection{Adiabatic Quantum  Annealing (AQC)ing}\label{ssec:quantum_annealing} 
The seminal work of Kadowaki and Nishimori  \cite{KadowakiNishimori1998} introduced quantum annealing. 
The authors argued that a quantum computer could be built based on the principle of finding the ground state of the Ising model under quantum fluctuations (causing state transitions) induced by a transverse magnetic field. 
The theoretical foundation of such a computational machine is  grounded on the adiabatic theorem \cite{BornFock1928}. 
In follow-up work, Farhi \textit{et al.}~\cite{Farhi2001} 
tested the quantum annealing algorithm simulated on a classical computer on hard instances of an $\mathcal{NP}$-complete problem. 
Although it is \textit{not} believed that quantum annealing can solve $\mathcal{NP}$-complete problems in polynomial time, for rugged energy landscapes with high and narrow spikes quantum annealing can be faster than simulated annealing \cite{das2005quantum,PhysRevX.6.031015} and for some instances D-Wave scales better than the classical path integral Monte Carlo method \cite{king2021scaling}.

AQCing can solve \textit{quadratic  unconstrained binary optimization} (QUBO) problems which read 
\begin{align}
    \min_{s \in \{ -1, 1 \}^n} s^\top J s+ b^\top s,  \label{eq:acq} 
\end{align} 
where $s$ is a binary vector of unknowns, $J$ is a symmetric matrix of inter-qubit couplings, and $b$ contains qubit biases. To change the binary variable to $x \in \{ 0, 1 \}^n$ one simply needs to insert $s_i=2x_i-1$ everywhere in \eqref{eq:acq}.
Quantum computers operate with \textit{qubits}, \textit{i.e.,}  
quantum mechanical systems which can be modeled with normalized  vectors $\ket \phi$ in a Hilbert space $\mathbb{H}$. 
One writes $\ket \phi = \alpha \ket 0 + \beta \ket 1$ with $\alpha, \beta \in \mathbb{C}$ and  $\left|\alpha\right|^2 + \left|\beta\right|^2 = 1$, where $\ket 0, \ket 1$ denote an orthonormal basis.
In contrast to classical bits, which are either in state $0$ or $1$, quantum mechanical systems can be in a superposition $\alpha \ket 0 + \beta \ket 1$, where after measurement in the so-called computational basis ($\{ \ket 0 ,  \ket 1 \}$) one obtains $\ket 0 $ with probability $ \abs{\alpha}^2$ and $\ket 1 $ with probability $ \abs{\beta}^2$. 
A composite system of, \textit{e.g.,}, two qubits can be described with vectors from the tensor product of the individual Hilbert spaces $\ket{\psi} \in \mathbb{H}  \otimes \mathbb{H}$. If two qubits can be described independently from each other, we have  $\ket{\phi} \in \mathbb{H}$ and $\ket{\eta} \in \mathbb{H}$, and the two-qubit state $\ket{\psi}$ is described by the product state $\ket{\psi}= \ket{\phi} \otimes \ket{\eta} $. Since the state space of $n$ qubits is $2^n$ dimensional,  a classical computer would have problems even storing all the coefficients for $n=50$ \cite{nielsen2002quantum}. A second ingredient required for quantum computing, in addition to the superposition of qubit states, is the interference between different possible computational paths in the $2^n$-dimensional state space caused by interactions between qubits. Thus, during the course of a quantum algorithm \textit{entangled} (non-classically correlated) states such as $ \frac{1}{\sqrt{2}}(\ket{0}\otimes \ket{0} + \ket{1}\otimes \ket{1})$ are created that cannot be decomposed in the above form. Entanglement between quantum states is often considered a necessary resource for quantum computing \cite{Briegel2001}. At the end of a quantum algorithm the coefficients $\alpha$ and $\beta$ for each qubit are determined in a suitable measurement leaving the (adiabatic or circuit-model) quantum computer in a classical (non-entangled) state.

Another central notion in quantum mechanics are linear \textit{Hamilton operators}, acting on elements of $\mathbb{H}$, which are used to describe the static and dynamic properties of a quantum system with the Schrödinger equation. The eigenstates of a (time-independent) Hamiltonian are the stationary energy states of a quantum system. The eigenvalues give the corresponding possible values of the system's energy.   
The key idea for solving problems like \eqref{eq:acq} is to prepare a quantum system in a known state of lowest energy of a Hamiltonian $H_I$, most commonly a product state of the form, 
\begin{equation}\label{eq:initialisation} 
    \ket{\psi(t=0)}= \bigotimes_{i=1}^{n} \frac{1}{\sqrt{2}}\left( \ket{0} +\ket{1} \right), 
\end{equation} 
with ``$\bigotimes$'' denoting the Kronecker product over $n$ subsystems, and each qubit is prepared in a superposition. Next, one constructs a \textit{problem Hamiltonian} $H_P$ in such a way that the lowest energy state of $H_P$ is a solution of \eqref{eq:acq}. Finally, one slowly switches from the initial Hamiltonian $H_I$ to the final Hamiltonian $H_P$ with, \textit{e.g.,} a linear combination
\begin{equation}\label{eq:Hamiltonian_transition} 
	H(t) = \Big(1 - \frac{t}{ \tau}\Big)\,H_I +  \frac{t}{\tau}\,H_P. 
\end{equation} 
If the transition in \eqref{eq:Hamiltonian_transition} is sufficiently slow (\textit{i.e.,} if $\tau$ is sufficiently large), and if the state of lowest energy (ground state) of $H(t)$ remains unique, the adiabatic theorem \cite{BornFock1928} guarantees that the quantum system will remain in the ground state for all $t$. Thus, the solution to \eqref{eq:acq} can simply be \textit{measured} after the time evolution starting in an eigenstate of $H_I$ (that is typically easy to prepare) and ending in an eigenstate of $H_P$. The theoretical requirements for adiabatic evolution  can be made precise in terms of the size of the \textit{spectral gap} of $H(t)$, \textit{i.e.,} the difference between the smallest and second smallest eigenvalues of $H(t)$, during the time evolution.  Building a system to actually prepare a quantum system and evolve it in an adiabatic way in practice, however, remains highly challenging. A spectral gap that is too small will lead to excitation of higher energy states during the annealing process, thus preventing the system from ending up in the desired ground state of $H_P$. Next, we discuss some details regarding such practical realizations. 
\subsection{Algorithm Design and Programming the D-Wave Quantum  Annealer}\label{ssec:programming_AQC} 
Every AQCing algorithm includes six steps, \textit{i.e.,} QUBO  preparation, minor embedding, %
quantum annealing (sampling), unembedding,  
bitstring selection (in the case of multiple annealings) and solution interpretation. 

\paragraph{QUBO  preparation:} Note that many computer vision problems are naturally  formulated in forms other than a QUBO. 
A lot of research thus focuses on the QUBO preparation,  \textit{i.e.,} how to formulate a target problem as a QUBO mappable to an AQCer (Sec.~\ref{sec:method} is devoted entirely to the QUBO preparation for Q-Match). 
Note that this step is performed entirely on the CPU. 
In QUBOs, each binary variable is interpreted as one \textit{logical} qubit in the quantum context. 
Thus, the biases $b$ and couplings between the qubits $J$ define a graph of logical problem qubits. 

\textbf{Minor embedding} is the mapping of logical qubits to the grid of physical hardware qubits of the AQCer. 
Due to the limited connectivity between the physical qubits, 
each logical qubit often requires several physical qubits in the minor embedding. 
Physical qubits realizing a single logical qubit---and having as similar quantum states during annealing as possible---build a \textit{chain}. 
To find a minor embedding for two arbitrary graphs is an $\mathcal{NP}$-hard optimization problem on its own \cite{Cai2014}. 
However the problem becomes easier since the one graph is fixed by the hardware and heuristic methods such as \cite{Cai2014} can be used. 
The core criteria for an optimal minor embedding is jointly minimizing the chain length and the total number of required physical qubits. 
\textbf{Quantum annealing} is initiated, once the minor embedding is complete. It corresponds to the evaluation of the system with respect to a time-dependent Hamiltonian such as \eqref{eq:Hamiltonian_transition} during which the quantum system ideally remains in the ground state during the entire annealing, see  Sec.~\ref{ssec:quantum_annealing}. In practice, due to such factors as 1) the expected qubits lifetime (\textit{e.g.,} influenced by interaction with the environment, which cannot be entirely prevented), and 2) a small spectral gap of $H(t)$, multiple anneals are needed, each of which leads to the global optimum with some probability ${<}1$. After each annealing, an \textbf{unembedding} algorithm assigns measured values to the logical qubits.

The final solution over multiple annealings is chosen in the \textbf{bitstring selection} step based on the occurrence frequency or the resulting energies.
Finally, the solution is \textbf{interpreted} in the context of the original problem and returned to the user. 
In this regard, the interpretation involves transformation of the  measured bitstring to the initial data modality (\textit{e.g.,}  permutation matrices). 
\begin{wrapfigure}{l}{0.16\textwidth}
\vspace{-14pt} 
  \begin{center}
    \includegraphics[width=0.16\textwidth]{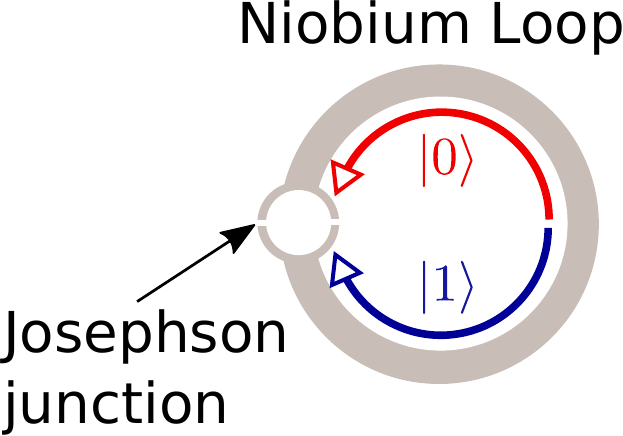}
  \end{center}
  \vspace{-10pt}
  \caption{Superconducting flux qubit.}
  \vspace{-17pt}
  \label{fig:niobium_loop} 
\end{wrapfigure}
D-Wave programming is done in Python, using the \textit{Leap 2} programming interface and remote access  tools \cite{DWave_Leap}. 
The D-Wave machine uses superconducting flux qubits \cite{mcgeoch2014adiabatic}. 
The qubits and the couplers are realized with loops of niobium, which have to be cooled to temperatures below 17mK. 
The current, that flows clockwise or anti-clockwise through the niobium loop can be modeled as a qubit, see Fig.~\ref{fig:niobium_loop}. 
Thus, it is possible to have a superposition between a current flowing clockwise and anticlockwise as long as no measurement is performed. 
The specified couplings $J$ and biases $b$ in \eqref{eq:acq} determine the time evolution of the magnetic fluxes during quantum annealing. 

\subsection{${\large \boldsymbol{\alpha}}$-Expansion}\label{ssec:alpha_expansion} 

The $\alpha$-expansion algorithm was introduced in \cite{boykov01graphcuts} to solve %
labeling optimization problems (\eg, for semantic segmentation or disparity estimation). A notable extension is \cite{FusionMove}, where the method is generalized to allow for parallel computation amongst other things.
In each iteration, a vector of binary variables $\alpha$ that correspond to two proposed values for the labels is computed so that the decision between the proposals improves the energy. \cite{boykov01graphcuts} finds the optimal expansion through a graph cut such that each pixel in an image is assigned to the first proposed label if $\alpha_i=1$ or the second proposed label if $\alpha=0$. 

Apart from this application, $\alpha$-expansion works particularly well in case the subproblem, \textit{i.e.,} the optimization for $\alpha$,
is submodular. Then the subproblem can be solved efficiently using graph cut techniques. 

\subsection{Cyclic Permutations}\label{ssec:cyclic_permutations} 
We define the set of permutation matrices as $\mathbb{P}_n = \{ X \in \{0,1\}^{n \times n}~|~ \sum_i X_{ij}=1,~\sum_j X_{ij}=1~ \forall i, j \}$.
A permutation matrix $X$ is called a \emph{k-cycle}, if there exist $k$ many disjoint indices $i_1, \hdots, i_k$ such that $X_{i_{1} i_k}=1$ and $X_{i_{j+1}i_j}=1$ for all $j=1, \hdots, k-1$, and $X_{ll} = 1$ for all $l \notin \{i_1, \hdots, i_k\}$, in which case $X = (i_1 ~ i_2 ~  \hdots ~ i_k)$ is a common notation. Two permutations $X$ and $X'$ are called \textit{disjoint} if $X_{ii}\neq 1$ implies $X'_{ii} = 1$ and $X'_{ii} \neq 1$ implies $X_{ii} = 1$. It is easy to show that disjoint permutations commute. Furthermore, it holds that any  $X$ can be written as the product of 2-cycles $c_i$, \textit{i.e.,} $X= \prod_{i=0}^N c_i$. Finally, we use the notation $X^1=X$ and $X^0 = I$ where $I$ is the identity for any matrix $X$.

\section{Method} \label{sec:method}

The main difficulty of previous quantum computing methods such as \cite{seelbach20quantum} for solving \eqref{eq:DASProblem} is handling the linear equality constraints arising from the optimization over permutation matrices. 
Instead of enforcing them via a penalty, the core idea of Q-Match is to minimize \eqref{eq:DASProblem} iteratively: We propose to use a variant of the $\alpha$-expansion algorithm that chooses $k$-many disjoint candidate cycles $c_i$ and uses quantum computing to solve the (still $\mathcal{NP}$-hard but smaller and unconstrained) subproblem of deciding whether to apply $c_i$ or not  (identified with $\alpha_i = 1$ and $\alpha_i=0$, respectively). 
In the following, we detail the idea of the cyclic $\alpha$-expansion as well as our proposed Q-Match algorithm.

\subsection{Cyclic ${\large \boldsymbol{\alpha}}$-Expansion} %
\label{subsec:cyclic_alpha_expansion} 

Let $C= \{c_1,..., c_m \}$ be a set of $m$ disjoint cycles. We consider the following optimization problem for an initial permutation $P_0$:
\begin{equation}
    \argmin_{ \{ P\in \mathbb{P}_n |  \exists  \mathbf{\alpha} \in  \{ 0,1\}^m: \ P = \left(\prod_i c_i^{\alpha_i} \right) P_0 \}  } E(P), \label{eq:IterationStep}
\end{equation}
where $E$ is defined in \eqref{eq:DASProblem} and 
$\alpha$ is a binary vector parametrizing $P$.  
As noted in Sec.~\ref{ssec:cyclic_permutations}, the order in the product of all $c_i^{\alpha_i}$ does not matter as disjoint permutations commute. %
\eqref{eq:IterationStep} would be equivalent to \eqref{eq:DASProblem} if $C$ was not fixed with disjoint cycles but could also be optimized. 

\vspace{-7pt}
\paragraph{Problem complexity:} The optimization problem in \eqref{eq:IterationStep} makes a binary decision for each cycle whether it should be applied.
In this section, we show that \eqref{eq:IterationStep} is an optimization problem with complexity dependent on the number of cycles and not $n$. 
We can convert the multiplication of cycles from \eqref{eq:IterationStep} into the following linear combination 

\begin{equation} 
P(\alpha) = P_0 + \sum_{i=1}^m \alpha_i(c_i - I) P_0,  \label{eq:Parametrization}
\end{equation} 
where $I$ is the identity. 
See the supplement for a proof.  
Next, we  use this representation to show that \eqref{eq:IterationStep} is a problem in the form of \eqref{eq:acq} with size $m$. Let $P,Q$ be two permutations and $E(Q, R) = \text{vec}(Q)^TW\text{vec}(R)$ which is linear in each component. 
We write $C_i = (c_i - I)P_0$ to simplify the notation.
Thus, to solve \eqref{eq:IterationStep}, we need to minimize 
\begin{equation*}\label{eq:E_expansion}  
\begin{aligned}
   &E( P_0 + \sum_i \alpha_i C_i, P_0 + \sum_j \alpha_j C_j ) \\
   &= E(P_0, P_0 + \sum_j \alpha_j C_j ) + \sum_i \alpha_i E(C_i, P_0+\sum_j \alpha_j C_j ) \\
   &= E(P_0, P_0) + \sum_j \alpha_j E(P_0, C_j) + \sum_i \alpha_i E(C_i,P_0) \\
   &+ \sum_i \sum_j \alpha_j \alpha_i E(C_i,C_j) = E(P_0, P_0 ) + \alpha^T \tilde{W}\alpha ,
\end{aligned}
\end{equation*} 

with
\begin{equation} 
\small
\tilde{W}_{ij} = \begin{cases}
E(C_i,C_j) & \text{if } i\neq j,\\
E(C_i,C_i) + E(C_i,P_0)+E(P_0, C_j) &\text{otherwise.}
\end{cases}
\label{eq:Couplings}
\end{equation} 
As $E(P_0,P_0)$ is constant w.r.t. $\alpha$, we are left with 
\begin{align}
     \min_{\alpha \in \{ 0,1 \}^m} \alpha^\top \Tilde{ W} \alpha, \label{eq:cyclicalpha}
\end{align}
which can be solved directly by AQCers. The interpretation of the optimal $\alpha$ is a binary indicator for whether to apply the corresponding cycle or not. With the basic methodology of \eqref{eq:cyclicalpha} being established, we proceed iteratively: Given a current estimate $P_{i-1}$ minimizing \eqref{eq:DASProblem}, we choose a set of disjoint cycles $C = \{ c_1, c_2, \dots, c_m \}$, optimize \eqref{eq:cyclicalpha}, and set $P_i = \left( \prod_j c_j^{\alpha_j}\right) P_{i-1}$. 

\subsection{Q-Match}\label{subsec:qmatch}

\begin{algorithm}[t]
\small
\SetKwInput{Input}{Input}
\SetKwInput{Output}{Output}
\SetKwRepeat{Do}{do}{while}
\DontPrintSemicolon

 \Input{Shapes $M, N$}
 \Output{Correspondences $P$} 
  Initialize $P$ (Sec.~\ref{ssec:initialization}) \\
  \Repeat{ $P$ does not change anymore }{ Calculate the $k$ worst vertices $I_M, I_N$ \eqref{eq:influence}\\
  Construct sub-matrix of worst matches $W_s$ (S.~\ref{subsec:subproblem})\\ 
  \Repeat{Every 2-cycle occurred in one set}{
  Choose random set of new 2-cycles  \\
  Calculate $\Tilde{W}$ (Sec.~\ref{subsec:subproblem} and \eqref{eq:Couplings})\\
  Solve the QUBO \eqref{eq:cyclicalpha} with AQCing\\
  Update permutation according to \eqref{eq:Parametrization}
  }
  Apply permutation to worst vertices in $P$ 
  }
  \Return{$P$}
 \caption{Q-Match}\label{alg:overview} 
\end{algorithm}

Cyclic $\alpha$-expansion can optimize any QAP up to a problem size that fits within the AQC, \textit{e.g.,} by iterating over random sets of cycles.
To solve larger problems of (isometric) shape correspondence between two 3D shapes $M$ and $N$, we propose \emph{Q-Match}. 
If both shapes are discretized with the same mesh of $n$ vertices, $W$ has the following form:
\begin{align}
    W_{i\cdot n+k,j\cdot n +l} = \vert d_M^{g}(i,j) - d_N^{g}(k,l) \vert \label{eq:isometricW}
\end{align}
where $i,j$ are vertices on $M$, $k,l$ are vertices on $N$, and $d^{g}(a,b)$ is the geodesic distance between two points on the same shape. 
Therefore, for two matches $(i,k), (j,l)$ $W_{i\cdot n+k,j\cdot n +l}$ describes how well the geodesic distance is preserved between the elements of this pair. 
The optimal solution is the correspondence for which the distance between all pairs of points is preserved the best, for isometries the optimal energy of \eqref{eq:DASProblem} is zero.
Since the geodesic distances have geometric meaning, this propagates into the QAP, and Q-Match leverages this to choose better sets of cycles. 
In Q-Match, the set of tested cycles is based on the contribution of each vertex to \eqref{eq:DASProblem} (Sec.~\ref{ssec:choosing_cycles}), we can solve subproblems iteratively (Sec.~\ref{subsec:subproblem}), and the initial permutation is descriptor-based (Sec.~\ref{ssec:initialization}). 
An overview of the entire Q-Match algorithm can be found in Alg.~\ref{alg:overview}.

\subsubsection{Choosing Cycles} \label{ssec:choosing_cycles} 

One option to select $C$ is by randomly drawing disjoint cycles, which might, however, take many iterations to reach the optimum. 
Instead we make use of the fact that the isometric shape matching problem creates QAPs in which the entries are highly correlated.

\paragraph{Worst Vertices:}
For a point with index $v$ in $M$ and a permutation $P$, we can quantify the influence of $v$ on \eqref{eq:DASProblem} via
\begin{align}
    I(v) = \sum_{w\in M} W_{v \cdot n + P(v), w \cdot n + P(w)},  \label{eq:influence}
\end{align}
where $P(v)$ denotes the index $v$ is mapped to. %
For vertices on $N$ we proceed equivalently. 
If $I(v)$ is high, this is an indicator for $P(v)$ being inconsistent with the majority of other matches in $P$. 
Therefore, we collect the $m$ vertices with highest values for $I$ on each shape in the two sets $I_M$ and $I_N$. 
The potential improvement of \eqref{eq:DASProblem} is maximized by this choice. 
See Sec.~\ref{subsec:convergence} for an analysis of the size of $m$.

\paragraph{2-Cycles:}
Subsequently, we choose a set of $k$ random but disjoint 2-cycles from $I_M, I_N$. While the limitation to 2-cycles might sound restrictive, the following lemma (proven in the appendix and in many abstract algebra textbooks \cite{scott2012group}) highlights the expressive power of disjoint 2-cycles: 

\begin{lemma}\label{lemma1}
Every permutation $P$ can be written as $P=Q R$, where $Q$ and $R$ are products of disjoint 2-cycles. %
\end{lemma}

\subsubsection{Subproblems}\label{subsec:subproblem}

3D meshes normally consist of more than thousand vertices, this means $W \in \mathbb{R}^{n^2 \times n^2}$ cannot be precomputed and stored in memory. 
Therefore, in each iteration all needed entries of $W$ to compute $\tilde W$ for \eqref{eq:cyclicalpha} have to be computed from scratch.
As can be seen in \eqref{eq:Couplings} several evaluations of the cost function $E$ are needed for $\tilde W$ which makes this operation expensive. 

To be more efficient, we split the computation into the calculation of $W_s$, a $k^2 \times k^2$ reduction of $W$ based on the $k$ worst vertices in $I_M, I_N$. 
This is the most expensive operation in our algorithm (see Fig.~\ref{fig:runtime}).
For a set of 2-cycles on $I_M, I_N$, computing $\tilde W$ is now efficient.
To reduce the number of times we have to compute $W_s$, instead of just choosing one set of 2-cycles on each $I_M, I_N$, we evaluate several sets. 
In practice, we sample sets of 2-cycles on $I_M, I_N$ until every possible 2-cycle has been chosen at least once.
A detailed description how to compute $W_s$ and  $\tilde W_s$ and a runtime analysis 
can be found in the supplementary material.

\subsubsection{Initialization} \label{ssec:initialization}

The only thing left is choosing $P_0$ for the first iteration. 
Instead of a random initialization, we calculate a descriptor-based similarity $S \in \mathbb{R}^{N \times N}$ between all vertices of $M$ and $N$, equivalent to the strategy of \cite{vestner2017efficient}. 
$S_{mn}$ contains the inner product of the normalized HKS \cite{sun09hks} and SHOT \cite{tombari2010unique} descriptors of vertices $m \in M, n \in N$, and we solve a linear assignment problem on $S$ to get the first permutation.

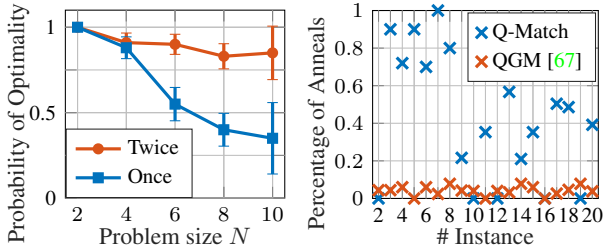
\begin{figure}
    \centering
    \input{FIGURES/probability.tikz}
    \input{FIGURES/percantagevsqgm.tikz}
    \caption{ Experiments on how often Q-Match retrieves the optimum for \eqref{eq:DASProblem}. (Left) Fraction of random instances with increasing problem size for which our iterative method leads to the optimum with brute-force calculation. (Right) Success rate of Q-Match and QGM \cite{seelbach20quantum} on $20$  random problems. We count (QGM) or lower bound (ours) the percentage of anneals that reach the optimum. }
    \label{fig:Benchmark}
\end{figure}

\section{Experiments}\label{sec:experiments} 
We evaluate accuracy, convergence, and runtime 
of our method on random QAP instances (Sec.~\ref{subsec:CompWithInserted}), FAUST (Sec.~\ref{subsec:Faust}) and QAPLIB (Sec.~\ref{subsec:qaplib}). 
Since quantum computing time is still extremely expensive, we only evaluate on subsets of FAUST and QAPLIB, but our results show how promising this technology can be in the future. 
All experiments were done on Intel i5 8265U CPU with 16GB RAM using Python  $3.8$ and D-Wave Advantage system $1.1$ accessed by Leap 2. 
Note that one can use classical QUBO solvers for the subproblems. In the appendix, we present experiments for this with a simulated annealing sampler.

\subsection{Comparison to Penalty Based Implementation} \label{subsec:CompWithInserted}
We compare Q-Match to the inserted method in \cite{seelbach20quantum} for problems of the form  \eqref{eq:isometricW} with distances $d_{i,j}$ drawn uniformly at random from $[0,1]$, and $W$ being constructed using a randomly drawn true permutation $P$. 
We compare the success probability of both methods on $20$ random instances, see Fig.~\ref{fig:Benchmark}-(right). 
Q-Match finds the optimum in all but $4$ cases.
For a fair comparison, we restrict the number of anneals per iteration in our method to $10$ and use $500$ anneals for the inserted method. 
Additionally, the success probability of the iterative method is bounded by the success probabilities of the individual steps multiplied with each other. 
In the case that multiple distinct outputs have the lowest energy, both results will be counted as success. 
For $n\in \{2,4,6,8,10 \}$, we count how often the optimal permutation is reached, if each subproblem is solved to global optimality.
For this we generated $20$ random instances for $n=10$ and $100$ random instances for the other dimensions. The errors are calculated as a binomial proportion confidence interval that should contain the true probability in $95\%$ of the cases. This experiment was performed first so that all 2-cycles occur only once and a second time where the whole process is carried out twice, see Fig.~\ref{fig:Benchmark}-(left). 
Moreover, we study if the solution with the lowest value returned by the quantum annealer is globally optimal. 
For $500$ runs and coupling matrices up to a size of $13$ this was always the case. 
Therefore, the number of optima found by Q-Match is identical to the probabilities in Fig.~\ref{fig:Benchmark}-(left). 

\begin{figure}
    \centering
    \input{FIGURES/faust_error.tikz}%
    \input{FIGURES/worst.tikz}
    \caption{Evaluation of cumulative error \cite{kim11bim} (left) and convergence (right) on FAUST. (Left) We compare against Simulated Annealing \cite{holzschuh20simanneal} without post-processing and Functional Maps \cite{ovsjanikov2012functional}. Dashed lines indicate non-quantum approaches. The results have symmetry-flipped solutions removed, these have a comparable final energy for all methods but are not recognized as correct in the evaluation. (Right) We show convergence of the energy over $30$ iterations. The larger the set of worst vertices, the faster our method converges. The dashed grey line shows the optimal energy. }
    \label{fig:convergence}\label{fig:faust}
\end{figure}
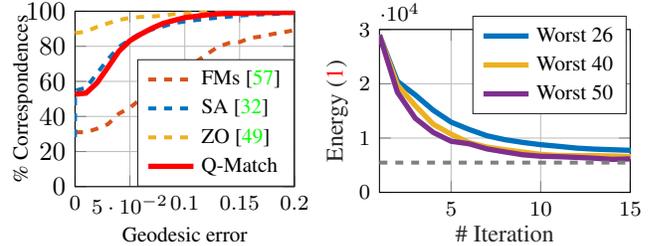

\begin{figure}
    \centering
    (i)
    \includegraphics[width=.18\linewidth]{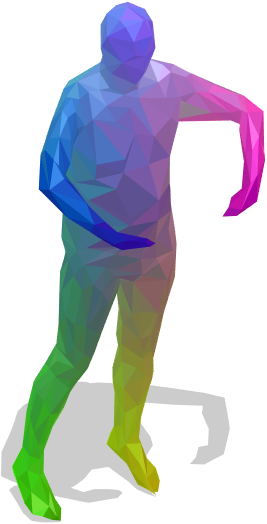}
    \includegraphics[width=.17\linewidth]{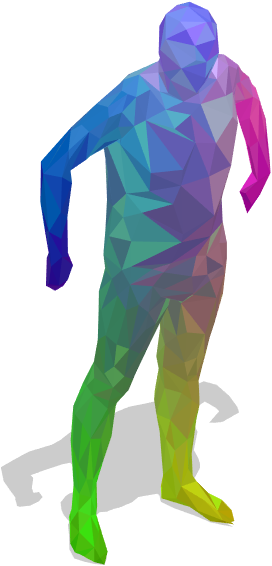}
    (ii)
    \includegraphics[width=.13\linewidth]{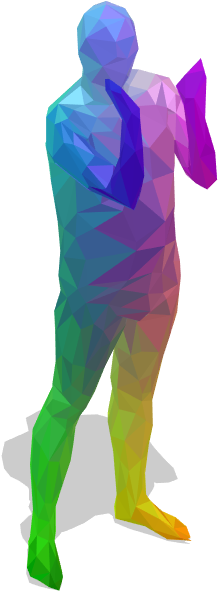}
    \includegraphics[width=.21\linewidth]{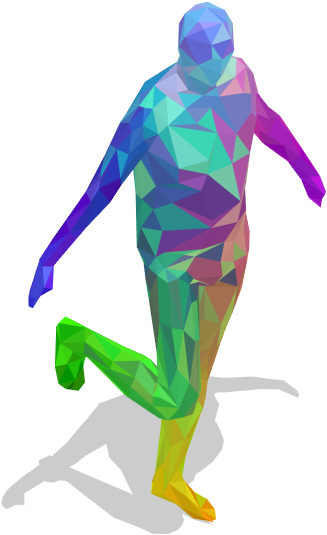}
    \caption{Example correspondences from the FAUST registrations. (i) Our results are overall correct with some outliers. (ii) A failure case. The solution partially swapped front and back. 
    These kinds of solutions are near-isometries, have energies close to the optimum, and often occur in all methods that purely solve \eqref{eq:DASProblem}. }
    \label{fig:examplesLow}
\end{figure}

\subsection{FAUST}\label{subsec:Faust}
We perform experiments on the registrations of the FAUST dataset \cite{bogo2014faust}. 
We evaluate on $5$ different pairs from one class within FAUST that were downsampled to $502$ vertices. 
It is in theory possible to use our method on the original resolution and more pairs, but QPUs are not as freely available as traditional hardware, and we focus our available computing time on a deeper analysis of our method. 
We apply Q-Match as described in Alg.~\ref{alg:overview} with 500 anneals, and compare to two non-quantum methods, Simulated Annealing \cite{holzschuh20simanneal} and Functional Maps \cite{ovsjanikov2012functional}. 
\cite{holzschuh20simanneal} is applied without ZoomOut \cite{melzi19zoomout} such that it is close to a version of our algorithm where only one 2-cycle is chosen in each iteration. The functional maps implementation includes a Laplacian commutativity regularizer and retrieves the final point-to-point correspondence through an LAP \cite{ovsjanikov2012functional}.
While \cite{seelbach20quantum} is the conceptionally closest method to ours, it only produces matches for up to preselected $4$ vertices. 
The quantitative evaluation can be seen in Fig.~\ref{fig:faust}, and some qualitative examples of our method in Fig.~\ref{fig:examplesLow}.
While we do not reach the same accuracy as the recent \cite{holzschuh20simanneal}, we do outperform the classical but still popular functional maps. 
Since AQCers are still newly developed and have limited complexity in comparison to traditional computers, we see this as a successful application of quantum annealing for real-world problems.

\subsection{Convergence} \label{subsec:convergence}

We evaluate how many iterations Q-Match needs to converge on FAUST. 
Fig.~\ref{fig:convergence} shows the convergence when choosing $26, 40, 50$ worst vertices in each iteration. Since the margin between $40$ and $50$ is small, we chose to run all further experiments with the $40$ worst matches.
In each iteration every 2-cycle occurs once in a set $C$. 
Some exemplary matches obtained in this way can be seen in Fig.~\ref{fig:examplesLow}. 

\subsection{Runtime} \label{subsec:runtime}

Fig.~\ref{fig:runtime} shows how the runtime of Q-Match scales with the problem size given by the number of worst vertices. 
The calculation of the submatrix $W_s$ is the most expensive. 
Computing the matrix $\Tilde{W}$ given $W_s$ takes only ms, while the calculation of $W_s$ is in the order of seconds. 
We did not optimize or parallelize this operation though.
The QPU access time is the time the quantum annealer spends to solve the problem, and is mostly independent of the problem size. 

We report results for about $25$ minutes of QPU time, which allows us to evaluate $10^4$ QUBOs. One FAUST instance using $40$ worst vertices and $500$ anneals per QUBO requires about $2$min of QPU time.
For all experiments we used the default annealing path and the default annealing time of $20 \mu$s. The chain strength was chosen as $1.0001$ times the largest absolute value of couplings or biases.

\begin{figure}
    \centering
    \input{FIGURES/t_qpu.tikz}%
    \quad\input{FIGURES/t_subproblem.tikz}
    \caption{Influence of the problem size on the runtime. We increase the size of the set of worst vertices (see Sec.~\ref{subsec:qmatch}) and measure the runtime development of different parts of our pipeline. While the QPU access time is mostly independent of the problem size, calculating $\tilde W$ on the CPU grows drastically. }
    \label{fig:runtime}
\end{figure}
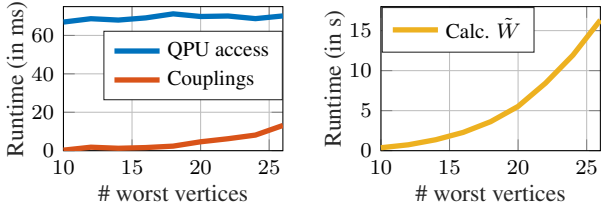

\subsection{QAPLIB}  \label{subsec:qaplib}

QAPLIB \cite{burkard97qaplib} provides a collection of QAPs together with the best known solution. 
This is the most commonly used benchmark for QAPs, with sets of problems from a variety of settings, \eg graphs but also unstructured data.
For some instances it is unclear if the best known solution is actually the optimum due to the complexity of the problem. 

We ran cyclic $\alpha$-expansion on some of the subsets ensuring that every 2-cycle occurs once when optimizing, and repeat this three times. We always choose the best solutions from $5000$ anneals if the size is $>25$, and $500$ anneals otherwise. The relative error we achieved is in Fig.~\ref{fig:Qaplibruntime_convergence}. 
For instances from \cite{burkard77qaplib, esc16qaplib, hadqaplib}, our solution is almost always within $1\%$ of the best solution.
For the $n=16$ instances in \cite{esc16qaplib}, we achieved the optimal solution in $9$ out of $10$ cases. 
Other sets contain instances up to size $30$ and pose a bigger challenge, our worst result was within $15\%$ of the optimum.
We report our exact solutions in the supplementary material.

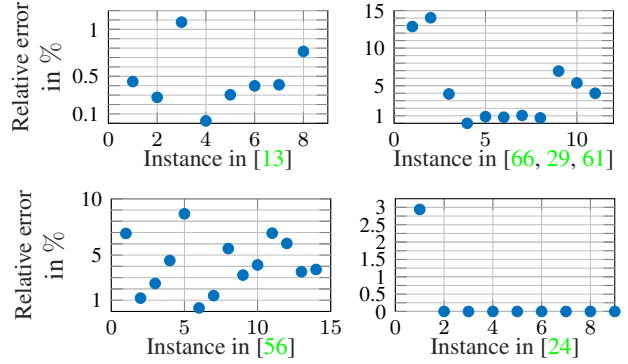
\begin{figure}
    \centering
    \input{FIGURES/qap_buof77.tikz}%
    \quad \input{FIGURES/qap_scrhadrou.tikz}
    \input{FIGURES/qap_nug.tikz}%
    \input{FIGURES/qap_esc16.tikz}%
    \caption{Relative error $ \frac{E_{\text{obtained}}-E_{\text{opt}}}{E_{\text{opt}}}$ of our method in  percentage for the instances of  \cite{burkard77qaplib}, %
    \cite{scrqaplib} (1-3), \cite{hadqaplib} (4-8) and \cite{rouqaplib} (9-11),%
    \cite{nugqaplib}, %
    and \cite{esc16qaplib} %
    in QAPLIB. 
    The problem sizes range between $12$ and $30$, of which \cite{nugqaplib} contains the larger ones where we do less well. 
     }
    \label{fig:Qaplibruntime_convergence}
\end{figure}

\section{Conclusion} \label{sec:conclusion} 
We introduce Q-Match and experimentally show that it can solve correspondence problems of sizes encountered in practical applications. 
This conclusion is made for the first time in literature for methods utilizing quantum annealing. %
Refraining from explicit permutation constraints and addressing the problem iteratively allowed us to significantly outperform the previous quantum state of the art, both in terms of the supported problem sizes and the probability to measure a globally optimal solution. We even achieved results comparable to a classical method.  
We hope that our work inspires further research in quantum annealing and related optimization problems in the vision community. 

{ 
\small 
\noindent\textbf{Acknowledgements.} 
This work was partially supported by the ERC  consolidator grant 4DReply (770784). 
}

{\small
\bibliographystyle{ieee_fullname}
\bibliography{egbib}
}

\onecolumn 
\newpage 
\twocolumn
\begin{center}
\textbf{{\Large Supplementary Material}} 
\end{center}
\appendix

This supplementary material provides 
a deeper analysis of the proposed Q-Match approach and more  experimental details. This includes: 
\begin{itemize}[leftmargin=*]\itemsep0em 
    \item Further analysis of the solution quality for individual  QUBOs and visualizations of the minor embeddings  (Sec.~\ref{sec:QUBO_Analysis}). 
    \item Derivation of Eq.~\eqref{eq:Parametrization} and proof of Lemma \ref{lemma1} from the main matter  (Sec.~\ref{sec:derivations_and_proofs}). 
    \item A toy example, in which all 2-cycles individually do not improve the energy (Sec.~\ref{sec:toy_example}). 
    \item More details on calculating $W_s$ and $\tilde W$ (Sec.~\ref{sec:W_calculations}). 
    \item A list of our solutions to selected QAPLIB problems \textit{vs} ground truth (Sec.~\ref{sec:exact_QAPLIB}). 
\end{itemize}

\section{Analysis of the Individual QUBOs}\label{sec:QUBO_Analysis} 
We analyze the quality of the solution of the individual QUBOs in dependence of the dimension. The success probability for one QUBO solution is defined as the fraction between the anneals that ended up in the optimum and the number of anneals. The success probability averaged over $20$ QUBOs per dimension at the first two iterations, \textit{i.e.,} computations of \eqref{eq:IterationStep} for a set of 2-cycles, of one instance of the FAUST dataset can be seen in Fig.~\ref{fig:successFaust}. %
We see that for increasing dimension the success probability is decreasing and less runs end up in the ground state. One possible way to reverse this trend would be to increase the annealing time, which we left constant at $T_A= 20\mu s$. We also plotted the fraction of QUBOs, where the ground state was among the returned solutions. Leaving the number of anneals constant this probability declined from $1$ for $4-24$ worst vertices to $0.4$ for $40$ worst vertices. To get the optimum for more instances we performed the experiment with $5000$ anneals for $40$ and $50$ worst vertices. In this experiment we found the optimum in $90\%$ or $45\% $ of the cases, respectively.

Note that quantum annealing is a stochastic algorithm. Therefore a success probability $P_s$ is directly linked with the amount of time needed to to get the optimum with, \textit{e.g.,} $99 \%$ probability: 
\begin{equation}
    T= \frac{\ln(1- 99\%) }{\ln(1-P_s )} T_A, \label{eq:TimeStochastic}
\end{equation}
where $T_A$ is the time for one anneal. 
We also computed the binomial proportion confidence interval for the probability that the optimum of the QUBO is found, with at least one anneal. If the underlying distribution is binomial, then the true probability lies within $20\%$ of our estimates in $95\%$ of cases. If the number of worst vertices is $23$ or less the estimate is even closer to the true probability.
The binomial confidence interval can be seen in Fig.~\ref{fig:successFaust}, on the right.
In the left plot of Fig.~\ref{fig:successFaust} the standard deviation of the success probability for averaging over the 20 individual QUBOs is also depicted. This shows that the success probability strongly varies for the distinct QUBO instances.

\begin{figure}
    \centering
    \input{FIGURES_SUPP/success.tikz}\input{FIGURES_SUPP/success2.tikz}
    \caption{Success probability (left) and the fraction of executions where the best solution is the optimum (right), for different problem sizes and number of anneals. The success probability decreases with the problem size, and, therefore, more anneals are necessary. In the left plot the deviation bar is the standard deviation and in the right plot it is the binomial proportion confidence interval. 
    No lower part of the deviation bar means it goes beyond zero in the left plot.} 
    \label{fig:successFaust}
\end{figure}
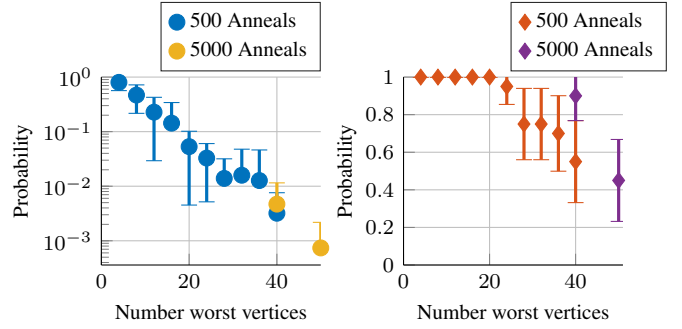

\subsection{Minor Embedding} 
In most cases, our problems result in fully connected  logical qubit graphs. 
In Figs.~\ref{fig:ConnectionA} and \ref{fig:ConnectionB}, there are illustrations of the minor embeddings computed by the method of Cai \textit{et al.}~\cite{Cai2014} (used in Leap 2) and visualized by the problem inspector of Leap 2 \cite{DWave_Leap}. 
We plot the average maximum chain lengths and average numbers of physical qubits in the obtained minor embeddings in Fig.~\ref{fig:emb}. 

\begin{figure*}
    \centering
    \includegraphics[width=0.24\textwidth]{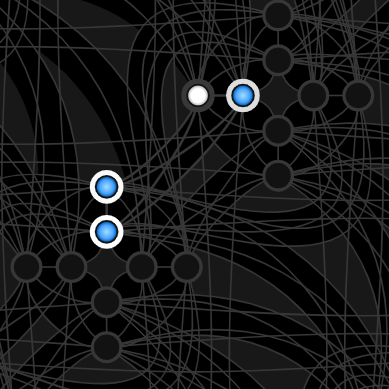}
      \includegraphics[width=0.24\textwidth]{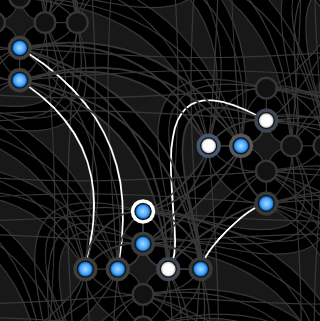}
    \includegraphics[width=0.24\textwidth]{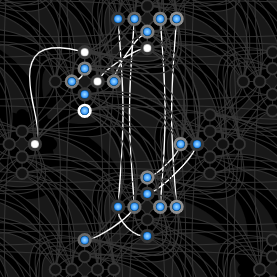}
        \includegraphics[width=0.24\textwidth]{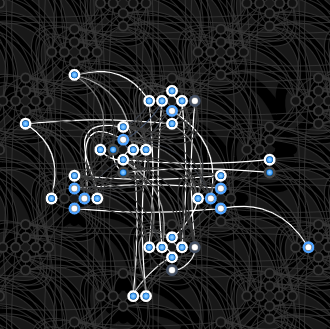}
          \caption{Illustration of the embedding from the D-Wave Leap 2 problem inspector \cite{DWave_Leap} for using $8,16,24,32$ worst vertices. One node depicts a physical qubits. The inner color shows the measured value in the lowest energy state, while the color of the outer ring shows the sign of the bias, \textit{e.g.,} the coefficient of the linear term in the optimization problem \eqref{eq:acq}. The edges depict the chains. 
    }
        \label{fig:ConnectionA}

\end{figure*}
\begin{figure*}
    \centering
          \includegraphics[width=0.45\textwidth]{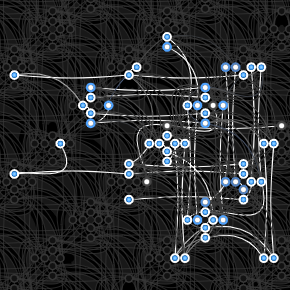}
          \includegraphics[width=0.45\textwidth]{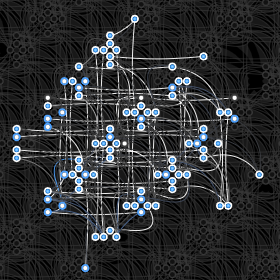}
    \caption{Illustration of the embedding from the D-Wave Leap 2 problem inspector \cite{DWave_Leap} for using $40$ and $50$ worst vertices. One node depicts a physical qubits. The inner color shows the measured value in the lowest energy state, while the color of the outer ring shows the sign of the bias, \textit{e.g.,} the coefficient of the linear term in the optimization problem \eqref{eq:acq}. The edges depict the chains. 
    }
       \label{fig:ConnectionB}
\end{figure*}

\begin{figure}
    \centering
    \input{FIGURES_SUPP/chainlength.tikz}
    \input{FIGURES_SUPP/qubits.tikz}
    \caption{Average maximal chain length and number of physical qubits for increasing problem size averaged over 4 instances. 
    The number of logical qubits increases linearly with the problem size.}

    \label{fig:emb}
\end{figure}
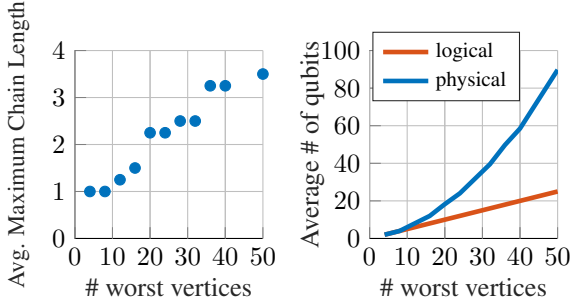

\section{Derivations and Proofs}\label{sec:derivations_and_proofs} 

\subsection{Derivation of Eq.~{\large \eqref{eq:Parametrization}} in Sec.~\ref{subsec:cyclic_alpha_expansion}}  %

It is stated in the main paper in Eq.~\eqref{eq:Parametrization}, that one can convert the multiplication of cycles from \eqref{eq:IterationStep} into an additive structure, \textit{i.e.,} that
\begin{align*} 
P(\alpha) &=\left(\prod_i c_i^{\alpha_i} \right) P_0= P_0 + \sum_{i=1}^m \alpha_i(c_i - I) P_0
\end{align*} 
holds, where $I$ is the identity. 

\begin{proof}
Consider the case where we only have a single  cycle $c$. Now, the following holds: 
\begin{align*}
P(\alpha) &= P_0 +  \alpha(c - I) P_0 = (1-\alpha)P_0 + \alpha c P_0\\
&= \begin{cases}
P_0,  &\text{ for } \alpha=0 \\ 
cP_0, &\text{ for } \alpha=1 
\end{cases} \; = \; c^\alpha P_0. 
\end{align*}
Independent of $P_0$, we can write:
\begin{equation}
    c^\alpha = (1-\alpha) I + \alpha c.
\end{equation}
Additionally we can write \eqref{eq:Parametrization} independent of $P_0$ by applying the inverse permutation from the right side:
\begin{align*} 
\prod_i c_i^{\alpha_i}  = I + \sum_{i=1}^m \alpha_i(c_i - I).
\end{align*}
Now as an induction step we apply $c_{m+1}^{\alpha_{m+1}}$ from the right:
\begin{align} 
&c_{m+1}^{\alpha_{m+1}} \left(\prod_i c_i^{\alpha_i} \right)\nonumber \\ &=c_{m+1}^{\alpha_{m+1}} (I + \sum_{i=1}^m \alpha_i(c_i - I)) \nonumber\\
&=((1-\alpha_{m+1}) I + \alpha_{m+1} c_{m+1}) (I + \sum_{i=1}^m \alpha_i(c_i - I))\nonumber \\
&=I + \sum_{i=1}^{m+1} \alpha_i(c_i - I)+\alpha_{m+1}(c_{m+1} - I)   \sum_{i=1}^{m} \alpha_i(c_i - I).  \label{eq:ProofPara}
\end{align}

 We want to use that for two disjoint cycles $c_k, c_l$, the equality $(c_k- I)(c_l- I) = 0$ holds. In all the places where $c_k$ has $0$ on the diagonal, $c_l$ has $1$, because they are disjoint. This leads to the fact that in the rows, where $(c_k- I)$ is non-zero, $(c_l- I)$ has zero columns or rows, respectively. Therefore, the last term in \eqref{eq:ProofPara} vanishes and the statement is proven.

\end{proof}

\subsection{Proof of Lemma \ref{lemma1}}

To prove Lemma \ref{lemma1}, we first show that the statement is correct for a $k$-cycle.

\begin{lemma}\label{lemmaS} 
Let $P$ be a k-cycle. Then, $P$ can be written as a product of $Q$ and $R$, where $Q$ and $R$ are permutations that only consist of disjoint 2-cycles. 
\end{lemma}
\begin{proof}
Without loss of generality, let $P= (1 \; 2 \; 3 ... k  )$ be the k-cycle. %
This is possible because rearranging rows and columns does not change the problem.
For even $k$, $Q$ and $R$ take the following form:
\begin{align*}
    Q&= (1\;2)(3\;k)(4 \;(k-1))\hdots\Big( \Big(1+\frac{k}{2}\Big) \;\Big(2+ \frac{k}{2}\Big)  \Big),  \\ 
    R&= (2 \;k)(3 \;(k-1))\hdots \Big( \frac{k}{2} \;  \Big(2+\frac{k}{2}\Big) \Big). 
\end{align*}
It can be easily checked that $P = QR$ holds in this case. 
For uneven $k$, one can choose the following $Q$ and $R$, and the same holds: 
\begin{align*}
    Q&= (1\;2)(3\;k)(4 \;(k-1))\hdots\Big( \Big(\frac{k+1}{2}\Big) \;\Big(1+ \frac{k+1}{2}\Big)  \Big), \\
    R&= (2 \;k)(3 \;(k-1))\hdots \Big( \frac{k-1}{2} \; \Big(1+\frac{k+1}{2}\Big) \Big). \\[-35pt]
\end{align*}
\end{proof}
Next, the following argument gives the proof for Lemma \ref{lemma1}. 
\begin{proof} 
One can first write the permutation $P$ in cycle notation. Then, we decompose each cycle individually as it was shown in Lemma \ref{lemmaS}. Note that according to Lemma \ref{lemmaS}, the decomposition of the $k$-cycle does not require additional elements in $Q$ or $R$ that do not occur in the cycle. 
\end{proof} 

Permutations like $Q$ or $R$ that only consist of 2-cycles are called involutions.
The fraction of permutations that are involutions has following asymptotic behavior \cite{flajolet2009analytic} (Prop.~VIII.2.): 
\begin{equation}
    \frac{I_n}{n!}= \frac{e^{-\frac{1}{4}}}{2 \sqrt{\pi n }}n^{-\frac{n}{2}} e^{\frac{n}{2} + \sqrt{n}}\left( 1 + O\left(\frac{1}{n^{\frac{1}{5}}}\right) \right). 
\end{equation}
Considering Lemma \ref{lemmaS}, a rough estimate is  $I_n^2 >n !$. 
For these permutations, in theory, one step of the cyclic  $\alpha$-expansion would suffice to obtain to the  identity  (which could be the optimum w.l.o.g.). 
\paragraph{Note on Classical Optimization:} 

If \eqref{eq:cyclicalpha} was submodular, it would be possible to efficiently solve it by converting it to a graph cut problem instead of using AQCing. 
According to Bach \cite{bach2019submodular}, a quadratic function $f(\mathbf{ x}):=  \mathbf{x}^T W \mathbf{x}$ is submodular, if and only if all non-diagonal elements of $W$ are non-positive. Since $\tilde{W}_{ij}$ can have a positive sign, the function is not submodular and, therefore, cannot be efficiently optimized. 
We tried small alternations of $\tilde{W}$ (\eg, changing the sign of the off-diagonal elements by switching $C_i$ to $-C_i$), but did not succeed in making  \eqref{eq:cyclicalpha} submodular. 
\subsection{Simulated Annealing \textit{\textbf{vs}} Quantum Annealing for the Subproblems}
We also performed Q-Match from Alg.~\ref{alg:overview} with a simulated annealing solver from \cite{DWave_Neal} for the subproblems. The quality of the results for the FAUST dataset and for QAPLIB can be seen in Figs.~\ref{fig:faust+sim} and \ref{fig:Qaplibruntime_convergence+Sim}. Here we executed the simulated annealing sampler with $5$ sweeps and performed $100$ runs. Increasing the number of sweeps further did not yield qualitatively different results.

In Fig.~\ref{fig:TimeSimulated} the processing time for the subproblems is plotted in dependence of the dimension. 
If one measures the wall clock time for a query to the D-Wave sampler in the same way one gets results of the order of seconds. This is due to the fact that the time the solver takes for the computation is overshadowed by the time to connect and get access to the machine. To get this time estimate one can simply look at \texttt{dwave ping} in ocean. For the QPU access time we already presented the results in dependence of dimension in Fig.~\ref{fig:runtime}. A more detailed overview of the different runtimes is given in \cite{DWave_Leap_Time}.

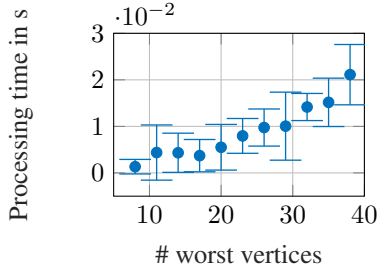
\begin{figure}
    \centering
    \input{FIGURES/timeSim.tex}
    \caption{Processing time of the simulated annealing sampler in dependence of the dimension. For one execution we have $100$ runs and $5$ sweeps as parameters. We averaged the measured time over $10$ executions. 
    }
    \label{fig:TimeSimulated}
\end{figure}

\begin{figure}
    \centering
    \input{FIGURES/time.tex}
    \caption{Time to compute the minor embedding in dependence of the dimension of the subproplem. The subproblems stem from the Q-Match algorithm applied to the Faust dataset. We averaged the measured time over $5$ executions. 
    }
    \label{fig:timeMinorEmbedding}
\end{figure}
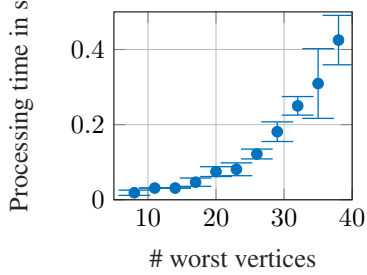

Since the minor embedding is computed locally one can directly measure the processing time in dependence of the dimension. In Fig.~\ref{fig:timeMinorEmbedding}, this measurement is averaged over 5 instances. 
The simulated annealing sampler takes here already an order of magnitude less time. However since we mostly want to embed a fully connected graph the problem to find an embedding could be solved beforehand and an existing embedding can be reused.

We also state that the emphasis of this work is not on benchmarking the subproblems. If one would do this one could also optimize over the annealing path and could use further features like the extended $J$-range for more precision and spin reversal transformations to mitigate some systematic errors.
Additionally one would also apply formula \eqref{eq:TimeStochastic}. In this section we only want to provide a rough overview of the computing time.
Preliminary work that does an in depth optimization of these algorithms confirms that quantum annealing is highly promising:
In \cite{PhysRevX.6.031015} QUBOs were found where quantum annealing yields a "time-to-$99\%$-
success-probability that is $\sim$ $10^8$
times faster than simulated annealing running on a single processor core". Reference \cite{junger2019performance} benchmarks the D-Wave 2000Q on a broader class of problems and confirms the potential of this technology.

\begin{figure}
    \centering
    \input{FIGURES/faust_error2.tikz} 
    \input{FIGURES/worst2.tikz}
    \caption{Evaluation of cumulative error \cite{kim11bim} (left) and convergence (right) on the FAUST dataset. (Left) We compare against Simulated Annealing \cite{holzschuh20simanneal} without postprocessing and Functional Maps \cite{ovsjanikov2012functional}. Dashed lines indicate non-quantum approaches. The results have symmetry-flipped solutions removed, these have an equivalent final energy for all three methods but are not recognized as correct in the evaluation protocol. (Right) We show the convergence of the energy over $30$ iterations. The larger the set of worst vertices, the faster our method converges. The dashed grey line shows the optimal energy. }
    \label{fig:faust+sim}
\end{figure}
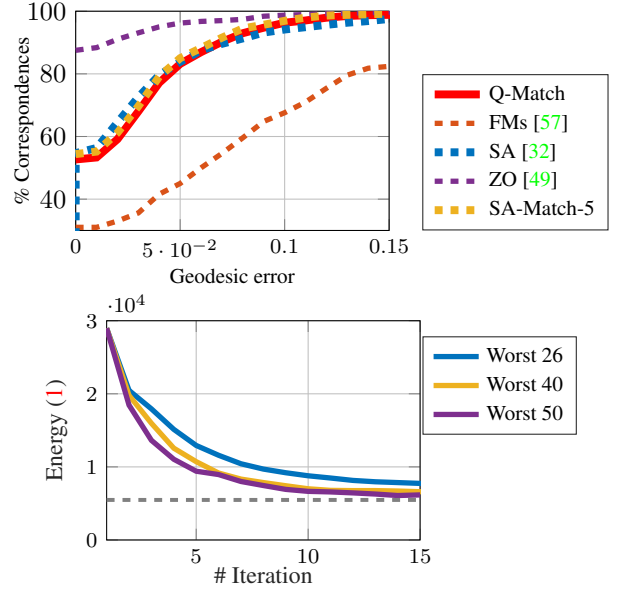

\begin{figure}
    \centering
    \input{FIGURES/qapbur.tex}%
  \input{FIGURES/qaprest.tex}
    \input{FIGURES/qapnug.tex}%
    \input{FIGURES/qapesc.tex}%
    \caption{Relative error $ \frac{E_{\text{obtained}}-E_{\text{opt}}}{E_{\text{opt}}}$ of our method in  percentage for the instances of  \cite{burkard77qaplib} (upper left),  \cite{scrqaplib} (1-3), \cite{hadqaplib} (4-8) and \cite{rouqaplib} (9-11)  (upper right), \cite{nugqaplib} (lower left) and \cite{esc16qaplib} (lower right) in QAPLIB. 
    The problem sizes range between $12$ and $30$, of which \cite{nugqaplib} contains the larger ones where we do less well. 
     }
    \label{fig:Qaplibruntime_convergence+Sim}
\end{figure}
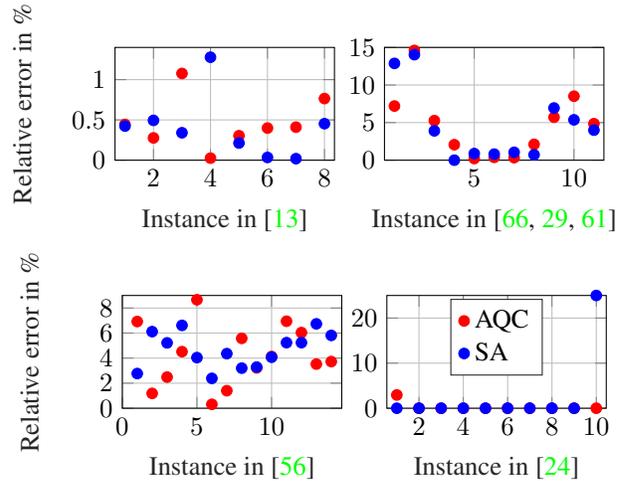

\begin{table*}[]
 \textbf{BO \cite{burkard77qaplib}:}   
    \normalsize
    \begin{center}
    \begin{tabular}{c||cccccccc}
         & a & b & c & d  & e & f & g & h\\
         \hline
        Optimum & 5426670 & 3817852 & 5426795 & 3821225 & 5386879 & 3782044 & 10117172 & 7098658 \\
        Our (Quantum) & 5450757 &  3828405 & 5485230 & 3822190 & 5403238 & 3797120 & 10158673 & 7152966 \\
        Our (Simulated Annealing) & 5449653	&3836716&	5445272&	3870067&	5398333&	3783329	&10119061&	7130826

    \end{tabular} \\
    \end{center}  
    \vspace{0.3cm} 
    \normalsize
    \textbf{ESC16 \cite{esc16qaplib}, \textbf{HAD} \cite{hadqaplib}:}  
    \normalsize
    \begin{center}
    \begin{tabular}{c||cccccccccc||ccccc}
         & a & b & c & d & e & f & g & h & i & j & a & b & c & d & e \\
         \hline
        Optimum & 68 & 292 & 160 & 16 & 28 & 0 & 26 & 996 & 14 & 8 & 1652 & 2724 & 3720 & 5358 & 6922 \\
        Our (Quantum) & 70 & 292 & 160 & 16 & 28 & 0 & 26 & 996 & 14 & 8 & 1652 & 2748 & 3750 & 5358 & 6922 \\
        Our (Simulated Annealing) &68&	292&	160&	16&	28&	0	&26&	996	&14&	10 &1686 &	2730	&3734	&5376	&7068

    \end{tabular} \\
    \end{center}
    \vspace{0.3cm} 
    \normalsize
    \textbf{NUG \cite{nugqaplib}: }  
    \small
    \begin{center}
    \begin{tabular}{c||cccccccccccccc}
         & a & b & c & d & e & f & g & h & i & j & k & l & m & n \\
         \hline
        Optimum & 578 & 1014 & 1610 & 1240 & 1732 & 1930 & 2570 & 2438 & 3596 & 3488 & 3744 & 5234 & 5166 & 6124 \\
        Our (Quantum) & 618 & 1026 & 1650 & 1296 & 1882 & 1936 & 2606 & 2574 & 3712 & 3632 & 4004 & 5550 & 5348 & 6352 \\
        Our (Simulated Annealing) &594	&1076	&1694	&1322	&1802	&1976	&2682	&2516	&3714	&3630	&3940	&5508	&5514	&6480
    \end{tabular} \\
    \end{center}
    \vspace{0.3cm} 
    \normalsize
    \textbf{SCR \cite{scrqaplib}, Rou \cite{rouqaplib}:  }
    \normalsize
    \begin{center}
    \begin{tabular}{c||ccc||ccc}
         & a & b & c & a & b & c  \\
         \hline
        Optimum & 31410 & 51140 & 110030 & 235528 & 354210 & 725522   \\
        Our (Quantum) & 35454 & 58320 & 114322 & 251872 & 373218 & 754506 \\
        Our (Simulated Annealing) &33672	&58606&	115822& 248982	&384354&	760738
    \end{tabular}
    \end{center}
    \normalsize
    \caption{Our solutions for exemplary sets of the QAPLIB dataset with different sizes of quadratic assignment problems. 
    }
    \label{tab:qaplib}
\end{table*}

\section{Failure Case for Individually Optimizing Over the 2-Cycles}\label{sec:toy_example} 

We present an example to proof that optimizing over larger sets of 2-cycles is superior to looking at single 2-cycles separately, as is done in \cite{holzschuh20simanneal}.
We construct a plane 2D-shape where the optimum can be reached with a collection of 2-cycles but each 2-cycle applied individually results in a worse energy starting from a specific permutation. 
Consider the points depicted in Fig.~\ref{fig:Points}. 
If we chose $a,b$ and $\epsilon$, we can compute the energies of permutations with \eqref{eq:isometricW} using Euclidean distances between the points. 
Possible values would be $b=10$, $a = 1$ and $\epsilon = 0.1$.

\begin{figure}
    \centering
    \includegraphics[scale=0.45]{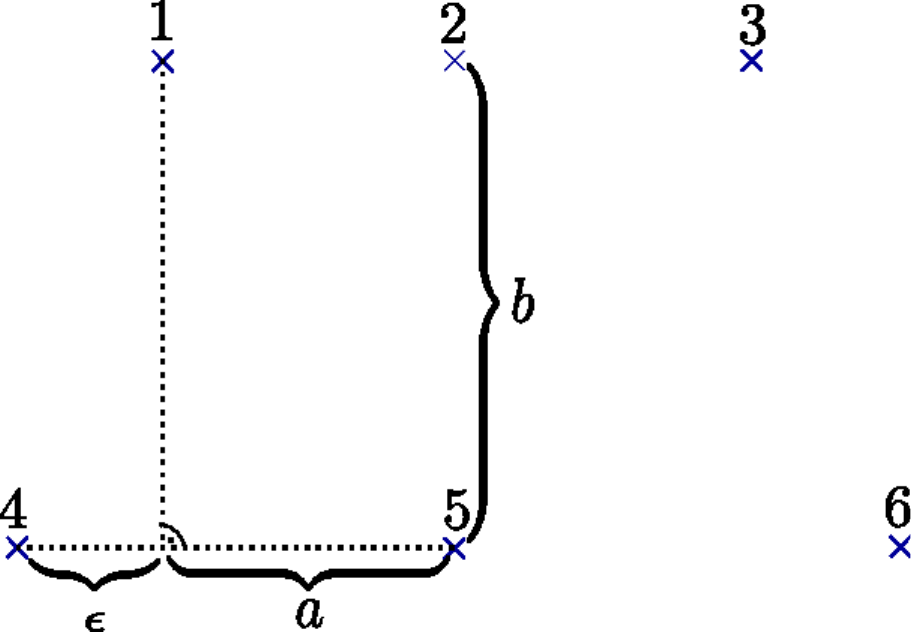}
    \caption{Exemplary shape to show that individually applying 2-cycles does not suffice. Because of the length $\epsilon$ the points are only nearly symmetric along an x-axis. The shape is invariant under the permutation $(13)(46)$.}
    \label{fig:Points}
\end{figure}

The shape is almost (but not exactly) symmetric with respect to mirroring along a shifted x-axis (Permutation (1 4)(2 5)(3 6)). 
In our experiment, the second shape is a copy of the first with permuted vertices, and we want to find the correspondence. 
Let the identity be the optimal solution, and the current permutation is $P_0= (1 4)(2 5)(3 6)$.
Permuting any of the three points on the upper $\{1,2,3\}$ to any of the lower points $\{4,5,6\}$ on the right causes -- despite the correct assignment -- a distortion of the (near-) isometry, such that no such 2-cycle improves the cost function when the assignment of the other points remains unchanged. However, applying all three correct 2-cycles at once, allows to pass to the global optimum with a lower energy. 

This illustrates that using our cyclic $\alpha$-expansion iteration step
for optimizing over multiple 2-cycles at a time can have significant advantages over a sequence of simple single 2-cycle updates.

\section{Calculation of $W_s$ and $\tilde W$}\label{sec:W_calculations} 

Notice that for a given permutation $P$ and a set of cycles $C$, it is possible to get $\tilde W$ without precomputing $W_s$ in roughly the same time as computing $W_s$. 
However, if $W_s$ is precomputed for a subset of vertices, $\tilde W$ can be computed very efficiently for any set of cycles on this subset.
Therefore, we once calculate the expensive $W_s$, and then evaluate several sets of cycles on it to increase the overall efficiency.

\subsection{Calculating {\large $W_s$}} 

If we want to solve a subproblem of \eqref{eq:DASProblem} we assume that all correspondences for indices, that are not optimized, stay fixed. 
Therefore, it is not sufficient to set $W_s$ to a submatrix of $W$, but we have to add the influence of these fixed correspondences. 
Given a set $s_M, s_N \subset \{1, \dots, n \}$ of indices which indicate the subproblem of $W$ (in Q-Match, these are the sets $I_M, I_N$), and a previous permutation $P$, we calculate $W_s$ as follows: 

\begin{align}
    (W_s)_{ikjl} = W_{ikjl} + \sum_{(v_M,v_N) \in V} W_{ikv_Mv_N} +  W_{v_Mv_Njl}. 
\end{align}
Here $V \subset M \times N$ is the set of correspondences indicated by a permutation $P$, with removed all tuples in $V$ which contain entries from $s_M, s_N$, \textit{i.e.,} $(v_M, v_N) \in V$ if $P(v_M) = v_N$ and $v_M \notin v_M, v_N \notin v_N$. 
This results in a $k^2 \times k^2$ matrix where each entry contains the sum of $\mathcal{O}(\vert C \vert )$ basic operations ($W_{ijkl} = \vert d(i,j) - d(k,l) \vert$, where all $d(\cdot, \cdot)$ are precomputed), resulting in $\mathcal{O}(k^4 \vert C \vert)$. 
The computation of each entry can be parallelized. 

\begin{figure}
    \centering
    \input{FIGURES/faust_error_badinit.tikz}
    \input{FIGURES/worst_badinit.tikz}
    \caption{Quantitative experiments comparing our method using the optimal descriptor initialization (solid lines) with the worst descriptor initialization (dashed lines). (Left) Cumulative geodesic error (left) and convergence (right) is shown on the FAUST dataset, otherwise equivalent to Fig.~\ref{fig:faust}. The dashed gray line is the ground truth value.  
    }
    \label{fig:badinit}
\end{figure}
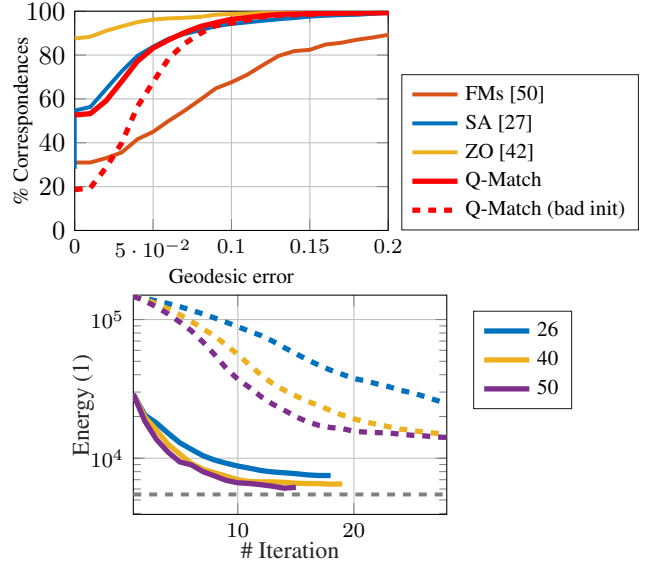

\subsection{Calculating {\large $\tilde W$}} 

Since we converted \eqref{eq:DASProblem} into a
QUBO \eqref{eq:acq}, $W_s$ also needs to be converted into $\tilde W$, \textit{i.e.,} the matrix describing the energy for the chosen combination of cycles. 
Since the cycles are sampled from $s_M, s_N$, $\tilde W$ can be computed from the entries of $W_s$, as we defined in \eqref{eq:Couplings} (and repeated here):
\begin{equation} 
\small
\tilde{W}_{ij} = \begin{cases}
E(C_i,C_j) & \text{if } i\neq j,\\
E(C_i,C_i) + E(C_i,P_0)+E(P_0, C_j) &\text{otherwise.}
\end{cases}\tag{7}
\label{eq:Couplings}
\end{equation} 
$E(C_i, C_j)$ can be calculated as two matrix-vector multiplications \eqref{eq:DASProblem}, however, since the vectors are vectorized permutation matrices with exactly $k$ non-zero entries, they can be written as two sums over $k$ entries.
This is a $m \times m$ matrix. 
Computing every entry separately leads to a complexity of $\mathcal{O}(m^2 k)$. 
In our setting with 2-cycles, $m = \frac{1}{2} k$ holds, therefore, we reach a complexity of $\mathcal{O}(k^3)$. Note that usually $\vert C \vert \gg k$, and calculating $\tilde W$ is a lot more efficient than calculating $W_s$ (see Fig.~\ref{fig:runtime}).

\section{Exact Solutions on QAPLIB}\label{sec:exact_QAPLIB} 

Since the relative error of the QAP is not invariant under shifts of $W$, we also report our results on QAPLIB in Table~\ref{tab:qaplib}. 
Here, it becomes clear again that we reach the optimum for  virtually all instances in ESC16 and HAD. 

\section{Non-Optimal Initialization}\label{sec:badinit}

Due to a sign error in our original experiments on the FAUST dataset, we ran them with the worst possible descriptor based initialization instead of the best. 
As expected the accuracy is not as high and the algorithm converges slower, but Q-Match does not break completely with a very bad initialization.
We see this as an indicator that finding high quality solutions for larger subproblems leads to a very robust pipeline.
See Fig.~\ref{fig:badinit} for the exact results.

\end{document}

%% file: FIGURES/probability.tikz
%
%
\definecolor{mycolor1}{rgb}{0.00000,0.44700,0.74100}%
\definecolor{mycolor2}{rgb}{0.85000,0.32500,0.09800}%
\begin{tikzpicture}
\begin{axis}[%
width=.35\linewidth,
height=.3\linewidth,
scale only axis,
xmin=1.5,
xmax=10.5,
ymin=0,
ymax=1.1,
xmajorgrids,
ymajorgrids,
axis background/.style={fill=white},
yticklabels={},
extra y ticks={0,0.25,0.5,0.75,1},
extra y tick labels={0,,0.5,,1},
y tick label style = {font=\footnotesize},
xticklabels={},
extra x ticks={2,4,6,8,10},
x tick label style = {font=\footnotesize},
xlabel style={font=\color{white!15!black}},
xlabel={\small Problem size $N$},
x label style={at={(axis description cs:0.5,0.1)}, anchor=north},
ylabel style={font=\color{white!15!black}},
ylabel={\small Probability of Optimality},
y label style={at={(axis description cs:0.3,.4)},rotate=0,anchor=south},
title style={font=\bfseries},
legend style={at={(0.58,0.0)}, anchor=south east, legend cell align=left, align=left, draw=white!15!black, font=\footnotesize}
]

\addplot+[color=mycolor2, mark options={scale=0.75}, line width=1.2pt, error bars/.cd, 
    y fixed,
    y dir=both, 
    y explicit,error bar style={line width=1.2pt,solid}]
  table[x=x, y=y,y error=error, row sep=crcr]{%
x   y       error \\
2	1       0\\
4	0.91    0.056091625043316 \\
6	0.9     0.0588\\
8	0.83    0.073624028686292\\
10	0.85    0.156493450342179 \\
};
\addlegendentry{Twice}

\addplot+[color=mycolor1, mark options={scale=0.75}, line width=1.2pt, error bars/.cd, 
    y fixed,
    y dir=both, 
    y explicit,error bar style={line width=1.2pt,solid}]
  table[x=x, y=y,y error=error, row sep=crcr]{%
x   y       error \\
2	1       0 \\
4	0.88    0.063692461092346 \\
6	0.55    0.097508768836449 \\
8	0.4     0.096019997917101 \\
10	0.35    0.209041144275475 \\
};
\addlegendentry{Once}

\end{axis}

\end{tikzpicture}%

%% file: FIGURES/percantagevsqgm.tikz
%
%
\definecolor{mycolor1}{rgb}{0.00000,0.44700,0.74100}%
\definecolor{mycolor2}{rgb}{0.85000,0.32500,0.09800}%
\begin{tikzpicture}

\begin{axis}[%
width=.36\linewidth,
height=.3\linewidth,
scale only axis,
xmin=1.5,
xmax=20.5,
ymin=0,
ymax=1,
xmajorgrids,
ymajorgrids,
axis background/.style={fill=white},
xlabel style={font=\color{white!15!black}},
xlabel={\small \# Instance},
xticklabels={},
extra x ticks={1,2,3,4,5,6,7,8,9,10,11,12,13,14,15,16,17,18,19,20},
extra x tick labels={,2,,4,,6,,8,,10,,12,,14,,16,,18,,20},
x label style={at={(axis description cs:0.5,0.1)}, anchor=north},
x tick label style = {font=\footnotesize},
ylabel style={font=\color{white!15!black}},
ylabel={\small Percentage of Anneals},
y label style={at={(axis description cs:0.26,.4)},rotate=0,anchor=south},
y tick label style = {font=\footnotesize},
title style={font=\bfseries},
legend style={at={(0.99,0.62)}, anchor=south east, legend cell align=left, align=left, draw=white!15!black, font=\footnotesize}
]
\addplot[only marks, mark=x, mark options={very thick}, mark size=3pt, draw=mycolor1] table[row sep=crcr]{%
x	y\\
1	1\\
2	0\\
3	0.9\\
4	0.72\\
5	0.9\\
6	0.7\\
7	1\\
8	0.8\\
9	0.216\\
10	0\\
11	0.3528\\
12	0\\
13	0.567\\
14	0.21\\
15	0.3528\\
16	0\\
17	0.504\\
18	0.486\\
19	0\\
20	0.392\\
};
\addlegendentry{Q-Match}

\addplot[only marks, mark=x, mark options={very thick}, mark size=3pt, draw=mycolor2] table[row sep=crcr]{%
x	y\\
1	0.02\\
2	0.044\\
3	0.044\\
4	0.06\\
5	0\\
6	0.06\\
7	0.024\\
8	0.078\\
9	0.042\\
10	0.04\\
11	0\\
12	0.04\\
13	0.032\\
14	0.078\\
15	0.06\\
16	0\\
17	0.026\\
18	0.046\\
19	0.078\\
20	0.038\\
};
\addlegendentry{QGM \cite{seelbach20quantum}}

\end{axis}

\end{tikzpicture}%

%% file: FIGURES/faust_error.tikz
%
%
\definecolor{mycolor1}{rgb}{0.00000,0.44700,0.74100}%
\definecolor{mycolor2}{rgb}{0.85000,0.32500,0.09800}%
\definecolor{mycolor3}{rgb}{0.92900,0.69400,0.12500}%
\definecolor{mycolor4}{rgb}{0.49400,0.18400,0.55600}%
\begin{tikzpicture}

\begin{axis}[%
width=.35\linewidth,
height=.28\linewidth,
at={(0.797in,0.617in)},
scale only axis,
xmin=0,
xmax=0.2,
ymin=0,
ymax=100,
xmajorgrids,
ymajorgrids,
axis background/.style={fill=white},
xlabel style={font=\color{white!15!black}},
xlabel={Geodesic error},
x label style = {font=\footnotesize},
x tick label style = {font=\footnotesize},
x label style={at={(axis description cs:0.5,0.05)}, anchor=north},
y label style = {font=\footnotesize},
ylabel style={font=\color{white!15!black}},
y label style = {font=\footnotesize},
ylabel={\% Correspondences},
y label style={at={(axis description cs:0.27,.5)},rotate=0,anchor=south},
title style={font=\bfseries},
legend style={at={(1,0.0)}, anchor=south east, legend cell align=left, align=left, draw=white!15!black, font=\footnotesize}
]
\addplot [color=mycolor2, dashed, line width=1.5pt]
  table[row sep=crcr]{%
0	31\\
0.01	31\\
0.02	33\\
0.03	35.6\\
0.04	41.6\\
0.05	45\\
0.06	50\\
0.07	54.6\\
0.08	59.6\\
0.09	64.8\\
0.1	67.6\\
0.11	71\\
0.12	75.4\\
0.13	79.6\\
0.14	81.8\\
0.15	82.4\\
0.16	84.8\\
0.17	85.6\\
0.18	87\\
0.19	88\\
0.2	89.2\\
0.21	89.8\\
0.22	90.4\\
0.23	91\\
0.24	91.4\\
0.25	91.8\\
0.26	92.2\\
0.27	92.6\\
0.28	92.8\\
0.29	93.2\\
0.3	93.2\\
0.31	93.6\\
0.32	94\\
0.33	94\\
0.34	94.6\\
0.35	94.8\\
0.36	95.4\\
0.37	95.8\\
0.38	95.8\\
0.39	96\\
0.4	96\\
0.41	96\\
0.42	96\\
0.43	96.4\\
0.44	97\\
0.45	97\\
0.46	97\\
0.47	97.2\\
0.48	97.2\\
0.49	97.2\\
0.5	97.2\\
};
\addlegendentry{FMs \cite{ovsjanikov2012functional}}

\addplot [color=mycolor1,dashed,line width=1.5pt]
  table[row sep=crcr]{%
0	28.1841638445412\\
0	54.6816479400749\\
0.01	56.3295880149813\\
0.02	64.4943820224719\\
0.03	72.5842696629214\\
0.04	79.625468164794\\
0.05	83.8202247191011\\
0.06	87.4906367041199\\
0.07	89.5880149812734\\
0.08	91.310861423221\\
0.09	93.1835205992509\\
0.1	94.1573033707865\\
0.11	94.9063670411985\\
0.12	95.6554307116105\\
0.13	96.3295880149813\\
0.14	96.8539325842697\\
0.15	97.5280898876404\\
0.16	97.9026217228464\\
0.17	98.2022471910112\\
0.18	98.3520599250936\\
0.19	98.8014981273408\\
0.2	99.1760299625468\\
0.21	99.5505617977528\\
0.22	99.8501872659176\\
0.23	100\\
0.24	100\\
0.25	100\\
0.26	100\\
0.27	100\\
0.28	100\\
0.29	100\\
0.3	100\\
0.31	100\\
0.32	100\\
0.33	100\\
0.34	100\\
0.35	100\\
0.36	100\\
0.37	100\\
0.38	100\\
0.39	100\\
0.4	100\\
0.41	100\\
0.42	100\\
0.43	100\\
0.44	100\\
0.45	100\\
0.46	100\\
0.47	100\\
0.48	100\\
0.49	100\\
0.5	100\\
};
\addlegendentry{SA \cite{holzschuh20simanneal}}

\addplot [color=mycolor3,line width=1.5pt,dashed]
  table[row sep=crcr]{%
0	87.5166002656043\\
0.01	88.3798140770252\\
0.02	91.0358565737052\\
0.03	93.0942895086321\\
0.04	95.0199203187251\\
0.05	96.2151394422311\\
0.06	96.7463479415671\\
0.07	97.0119521912351\\
0.08	97.4103585657371\\
0.09	98.406374501992\\
0.1	98.738379814077\\
0.11	98.937583001328\\
0.12	99.203187250996\\
0.13	99.33598937583\\
0.14	99.535192563081\\
0.15	99.667994687915\\
0.16	99.800796812749\\
0.17	99.867197875166\\
0.18	99.867197875166\\
0.19	99.867197875166\\
0.2	99.867197875166\\
0.21	99.867197875166\\
0.22	99.933598937583\\
0.23	99.933598937583\\
0.24	99.933598937583\\
0.25	99.933598937583\\
0.26	99.933598937583\\
0.27	100\\
0.28	100\\
0.29	100\\
0.3	100\\
0.31	100\\
0.32	100\\
0.33	100\\
0.34	100\\
0.35	100\\
0.36	100\\
0.37	100\\
0.38	100\\
0.39	100\\
0.4	100\\
0.41	100\\
0.42	100\\
0.43	100\\
0.44	100\\
0.45	100\\
0.46	100\\
0.47	100\\
0.48	100\\
0.49	100\\
0.5	100\\
};
\addlegendentry{ZO \cite{melzi19zoomout}}

\addplot [color=red,line width=2.0pt]
  table[row sep=crcr]{%
0	52.6693227091633\\
0.01	53.3067729083665\\
0.02	59.1235059760956\\
0.03	68.00796812749\\
0.04	77.1713147410359\\
0.05	83.1872509960159\\
0.06	86.9322709163347\\
0.07	90.2390438247012\\
0.08	92.9880478087649\\
0.09	94.7410358565737\\
0.1	96.3346613545817\\
0.11	97.0517928286853\\
0.12	97.9282868525896\\
0.13	98.4860557768924\\
0.14	98.7250996015936\\
0.15	98.8446215139442\\
0.16	99.0438247011952\\
0.17	99.1633466135458\\
0.18	99.2430278884462\\
0.19	99.2828685258964\\
0.2	99.3227091633466\\
0.21	99.4422310756972\\
0.22	99.601593625498\\
0.23	99.8406374501992\\
0.24	99.9601593625498\\
0.25	99.9601593625498\\
0.26	99.9601593625498\\
0.27	99.9601593625498\\
0.28	99.9601593625498\\
0.29	99.9601593625498\\
0.3	99.9601593625498\\
0.31	99.9601593625498\\
0.32	99.9601593625498\\
0.33	99.9601593625498\\
0.34	99.9601593625498\\
0.35	99.9601593625498\\
0.36	99.9601593625498\\
0.37	99.9601593625498\\
0.38	99.9601593625498\\
0.39	99.9601593625498\\
0.4	99.9601593625498\\
0.41	99.9601593625498\\
0.42	99.9601593625498\\
0.43	99.9601593625498\\
0.44	99.9601593625498\\
0.45	99.9601593625498\\
0.46	99.9601593625498\\
0.47	99.9601593625498\\
0.48	99.9601593625498\\
0.49	99.9601593625498\\
0.5	99.9601593625498\\
0.51	99.9601593625498\\
0.52	99.9601593625498\\
0.53	99.9601593625498\\
0.54	99.9601593625498\\
0.55	99.9601593625498\\
0.56	99.9601593625498\\
0.57	99.9601593625498\\
0.58	99.9601593625498\\
0.59	99.9601593625498\\
0.6	99.9601593625498\\
0.61	99.9601593625498\\
0.62	99.9601593625498\\
0.63	99.9601593625498\\
0.64	99.9601593625498\\
0.65	100\\
0.66	100\\
0.67	100\\
0.68	100\\
0.69	100\\
0.7	100\\
0.71	100\\
0.72	100\\
0.73	100\\
0.74	100\\
0.75	100\\
0.76	100\\
0.77	100\\
0.78	100\\
0.79	100\\
0.8	100\\
0.81	100\\
0.82	100\\
0.83	100\\
0.84	100\\
0.85	100\\
0.86	100\\
0.87	100\\
0.88	100\\
0.89	100\\
0.9	100\\
0.91	100\\
0.92	100\\
0.93	100\\
0.94	100\\
0.95	100\\
0.96	100\\
0.97	100\\
0.98	100\\
0.99	100\\
1	100\\
};
\addlegendentry{Q-Match}

\end{axis}

\end{tikzpicture}%

%% file: FIGURES/worst.tikz
%
%
\definecolor{mycolor1}{rgb}{0.00000,0.44700,0.74100}%
\definecolor{mycolor2}{rgb}{0.85000,0.32500,0.09800}%
\definecolor{mycolor3}{rgb}{0.92900,0.69400,0.12500}%
\definecolor{mycolor4}{rgb}{0.49400,0.18400,0.55600}%
\begin{tikzpicture}

\begin{axis}[%
width=.4\linewidth,
height=.26\linewidth,
at={(0.5in,0.1in)},
scale only axis,
xmin=1,
xmax=15,
ymin=0,
ymax=30000,
xmajorgrids,
ymajorgrids,
axis background/.style={fill=white},
xlabel style={font=\color{white!15!black}},
xlabel={\small \# Iteration},
x tick label style = {font=\footnotesize},
y tick label style = {font =\footnotesize},
x label style={at={(axis description cs:0.5,0.1)}, anchor=north},
ylabel style={font=\color{white!15!black}},
ylabel={\small Energy \eqref{eq:DASProblem}},
y label style={at={(axis description cs:0.28,.5)},rotate=0,anchor=south},
title style={font=\bfseries},
legend style={at={(0.986,0.48)}, anchor=south east, legend cell align=left, align=left, draw=white!15!black, font=\footnotesize}
]
\addplot [color=mycolor1, line width=2pt]
  table[row sep=crcr]{%
1	28967.7740945007\\
2	20462.8711526334\\
3	17938.8933045271\\
4	15133.3969379716\\
5	12933.9466254587\\
6	11590.26628047\\
7	10428.6223381632\\
8	9700.66895599068\\
9	9200.43498568485\\
10	8782.83810686946\\
11	8478.04747317207\\
12	8149.97998745927\\
13	7964.23038636745\\
14	7853.93685846935\\
15	7741.44210339278\\
16	7585.12901476899\\
17	7506.80118382658\\
18	7506.80118382658\\
};
\addlegendentry{Worst 26}

\addplot [color=mycolor3, line width=2.0pt]
  table[row sep=crcr]{%
1	28967.7740945007\\
2	19785.5638124038\\
3	15965.625080292\\
4	12559.5330345494\\
5	10716.6133549037\\
6	9161.67187490522\\
7	8331.93259485467\\
8	7855.8441526553\\
9	7422.46809373534\\
10	7005.97813927963\\
11	6789.73765964623\\
12	6755.55681831424\\
13	6758.7013013385\\
14	6698.82219891283\\
15	6605.60138840007\\
16	6566.61404866082\\
17	6562.50554973589\\
18	6501.10204693944\\
19	6525.49124386912\\
};
\addlegendentry{Worst 40}

\addplot [color=mycolor4, line width=2.0pt]
  table[row sep=crcr]{%
1	28967.7740945007\\
2	18455.8661354276\\
3	13636.3958207801\\
4	11043.0417990707\\
5	9409.15456174472\\
6	8950.78079372073\\
7	8002.40795811593\\
8	7449.97648451161\\
9	6921.51806768985\\
10	6651.48589420183\\
11	6574.70936789903\\
12	6441.67271594747\\
13	6285.61209839924\\
14	6082.44440415722\\
15	6166.50557865852\\
};
\addlegendentry{Worst 50}

\addplot [color=gray, dashed, line width = 1.5pt]
  table[row sep=crcr]{%
1	5479.47419680382\\
28	5479.47419680382\\
};

\end{axis}

\end{tikzpicture}%

%% file: FIGURES/t_qpu.tikz
%
%
\definecolor{mycolor1}{rgb}{0.00000,0.44700,0.74100}%
\definecolor{mycolor2}{rgb}{0.85000,0.32500,0.09800}%
\definecolor{mycolor3}{rgb}{0.92900,0.69400,0.12500}%
\definecolor{mycolor4}{rgb}{0.49400,0.18400,0.55600}%
\begin{tikzpicture}

\begin{axis}[%
width=.35\linewidth,
height=.23\linewidth,
scale only axis,
xmin=10,
xmax=26,
ymin=0,
ymax=75,
xmajorgrids,
ymajorgrids,
axis background/.style={fill=white},
y tick label style = {font=\footnotesize},
x tick label style = {font=\footnotesize},
xlabel style={font=\color{white!15!black}},
xlabel={\small \# worst vertices},
x label style={at={(axis description cs:0.5,0.1)}, anchor=north},
ylabel style={font=\color{white!15!black}},
ylabel={\small Runtime (in ms)},
y label style={at={(axis description cs:0.3,.5)},rotate=0,anchor=south},
title style={font=\bfseries},
legend style={at={(1,0.35)}, anchor=south east, legend cell align=left, align=left, draw=white!15!black, font=\footnotesize}
]
\addplot [color=mycolor1, line width=2pt]
  table[row sep=crcr]{%
10	66.96275\\
12	68.704125\\
14	67.994125\\
16	69.13325\\
18	71.229875\\
20	69.854875\\
22	70.109625\\
24	68.741\\
26	70.150625\\
};
\addlegendentry{QPU access}

\addplot [color=mycolor2, line width = 2pt]
  table[row sep=crcr]{%
10	0.26792\\
12	1.806795\\
14	1.243475\\
16	1.61332\\
18	2.36696\\
20	4.62347\\
22	6.21921\\
24	8.118925\\
26	13.121215\\
};
\addlegendentry{Couplings}

\end{axis}

\end{tikzpicture}%

%% file: FIGURES/t_subproblem.tikz
%
%
\definecolor{mycolor1}{rgb}{0.00000,0.44700,0.74100}%
\definecolor{mycolor3}{rgb}{0.92900,0.69400,0.12500}%
\definecolor{mycolor4}{rgb}{0.49400,0.18400,0.55600}%
\begin{tikzpicture}

\begin{axis}[%
width=.35\linewidth,
height=.23\linewidth,
scale only axis,
xmin=10,
xmax=26,
ymin=0,
ymax=18,
xmajorgrids,
ymajorgrids,
axis background/.style={fill=white},
y tick label style = {font=\footnotesize},
x tick label style = {font=\footnotesize},
xlabel style={font=\color{white!15!black}},
xlabel={\small \# worst vertices},
x label style={at={(axis description cs:0.5,0.1)}, anchor=north},
ylabel style={font=\color{white!15!black}},
ylabel={\small Runtime (in s)},
y label style={at={(axis description cs:0.3,.5)},rotate=0,anchor=south},
title style={font=\bfseries},
legend style={at={(0.7,0.68)}, anchor=south east, legend cell align=left, align=left, draw=white!15!black, font=\footnotesize}
]
\addplot [color=mycolor3, line width = 2pt]
  table[row sep=crcr]{%
10	0.366596221923828\\
12	0.735301017761231\\
14	1.36912727355957\\
16	2.28434896469116\\
18	3.62363886833191\\
20	5.53341674804688\\
22	8.43944239616394\\
24	11.9365055561066\\
26	16.319833278656\\
};
\addlegendentry{Calc. $\tilde W$}

\end{axis}
\end{tikzpicture}%

%% file: FIGURES/qap_buof77.tikz
%
%
\definecolor{mycolor1}{rgb}{0.00000,0.44700,0.74100}%
\begin{tikzpicture}

\begin{axis}[%
width=.35\linewidth,
height=.18\linewidth,
at={(0.758in,0.481in)},
scale only axis,
xmin=0,
xmax=9,
ymin=0,
ymax=1.2,
grid=both,
grid style={line width=.1pt, draw=gray!50},
yticklabels={},
extra y ticks={0.1, 0.2, 0.3, 0.4, 0.5, 0.6,0.7,0.8,0.9,1,1.1,1.2},
extra y tick labels={0.1,,,,0.5,,,,,1,,},
xlabel style={font=\color{white!15!black}},
x tick label style = {font=\footnotesize},
y tick label style = {font=\footnotesize},
x label style={at={(axis description cs:0.5,0.2)}, anchor=north},
xlabel={\small Instance in \cite{burkard77qaplib}},
ylabel style={font=\color{white!15!black}, align= center},
ylabel={\small Relative error\\
in $\%$},
y label style={at={(axis description cs:0.25,.5)},rotate=0,anchor=south},
axis background/.style={fill=white},
]
\addplot[only marks, mark=*, mark options={fill=mycolor1}, fill, mark size=2.0pt, draw=mycolor1] table[row sep=crcr]{%
x	y\\
1	0.443863363720287\\
2	0.276411971967483\\
3	1.07678657476467\\
4	0.0252536817381799\\
5	0.30368233628415\\
6	0.398620428530183\\
7	0.410203562813805\\
8	0.765046012922443\\
};

\end{axis}
\end{tikzpicture}%

%% file: FIGURES/qap_scrhadrou.tikz
%
%
\definecolor{mycolor1}{rgb}{0.00000,0.44700,0.74100}%
\begin{tikzpicture}

\begin{axis}[%
width=.35\linewidth,
height=.18\linewidth,
at={(0.758in,0.481in)},
scale only axis,
xmin=0,
xmax=12,
ymin=0,
ymax=15,
grid=both,
grid style={line width=.1pt, draw=gray!50},
yticklabels={},
extra y ticks={1,2,3,4,5,6,7,8,9,10,11,12,13,14,15},
extra y tick labels={1,,,,5,,,,,10,,,,,15},
xlabel style={font=\color{white!15!black}},
x tick label style = {font=\footnotesize},
y tick label style = {font=\footnotesize},
x label style={at={(axis description cs:0.5,0.2)}, anchor=north},
xlabel={\small Instance in \cite{scrqaplib,hadqaplib,rouqaplib}},
ylabel style={font=\color{white!15!black}},
y label style={at={(axis description cs:0.25,.5)},rotate=0,anchor=south},
axis background/.style={fill=white},
]
\addplot[only marks, mark=*, mark options={fill=mycolor1}, fill, mark size=2.0pt, draw=mycolor1] table[row sep=crcr]{%
x	y\\
1	12.8748806112703\\
2	14.0398904966758\\
3	3.90075433972552\\
4	0\\
5	0.881057268722474\\
6	0.806451612903225\\
7	1.04516610675625\\
8	0.722334585379958\\
9	6.93930233348052\\
10	5.36630812230032\\
11	3.99491676337864\\
};

\end{axis}
\end{tikzpicture}%

%% file: FIGURES/qap_nug.tikz
%
%
\definecolor{mycolor1}{rgb}{0.00000,0.44700,0.74100}%
\begin{tikzpicture}

\begin{axis}[%
width=.35\linewidth,
height=.18\linewidth,
at={(0.758in,0.481in)},
scale only axis,
xmin=0,
xmax=15,
ymin=0,
ymax=10,
grid=both,
grid style={line width=.1pt, draw=gray!50},
yticklabels={},
extra y ticks={1,2,3,4,5,6,7,8,9,10},
extra y tick labels={1,,,,5,,,,,10},
xlabel style={font=\color{white!15!black}},
x tick label style = {font=\footnotesize},
y tick label style = {font=\footnotesize},
x label style={at={(axis description cs:0.5,0.2)}, anchor=north},
xlabel={\small Instance in \cite{nugqaplib}},
ylabel style={font=\color{white!15!black}, align= center},
ylabel={\small Relative error \\
in $\%$},
y label style={at={(axis description cs:0.25,.5)},rotate=0,anchor=south},
axis background/.style={fill=white},
]
\addplot[only marks, mark=*, mark options={fill=mycolor1}, fill, mark size=2.0pt, draw=mycolor1] table[row sep=crcr]{%
x	y\\
1	6.92041522491349\\
2	1.18343195266273\\
3	2.48447204968945\\
4	4.51612903225806\\
5	8.66050808314087\\
6	0.310880829015536\\
7	1.40077821011673\\
8	5.57834290401968\\
9	3.2258064516129\\
10	4.12844036697249\\
11	6.94444444444444\\
12	6.03744745892243\\
13	3.52303523035231\\
14	3.72305682560419\\
};

\end{axis}
\end{tikzpicture}%

%% file: FIGURES/qap_esc16.tikz
%
%
\definecolor{mycolor1}{rgb}{0.00000,0.44700,0.74100}%
\begin{tikzpicture}

\begin{axis}[%
width=.35\linewidth,
height=.18\linewidth,
at={(0.758in,0.481in)},
scale only axis,
xmin=0,
xmax=9,
ymin=0,
ymax=3.25,
grid=both,
grid style={line width=.1pt, draw=gray!50},
xlabel style={font=\color{white!15!black}},
yticklabels={},
extra y ticks={0, 0.25, 0.5, 0.75, 1, 1.25,1.5,1.75,2,2.25,2.5,2.75,3},
extra y tick labels={0,,0.5,,1,,1.5,,2,,2.5,,3},
x tick label style = {font=\footnotesize},
y tick label style = {font=\footnotesize},
x label style={at={(axis description cs:0.5,0.2)}, anchor=north},
xlabel={\small Instance in \cite{esc16qaplib}},
ylabel style={font=\color{white!15!black}},
y label style={at={(axis description cs:0.25,.5)},rotate=0,anchor=south},
axis background/.style={fill=white},
]
\addplot[only marks, mark=*, mark options={fill=mycolor1}, fill, mark size=2.0pt, draw=mycolor1] table[row sep=crcr]{%
x	y\\
1	2.94117647058822\\
2	0\\
3	0\\
4	0\\
5	0\\
6	0\\
7	0\\
8	0\\
9	0\\
10	0\\
};

\end{axis}
\end{tikzpicture}%

%% file: FIGURES_SUPP/success.tikz
%
%
\definecolor{mycolor1}{rgb}{0.00000,0.44700,0.74100}%
\definecolor{mycolor2}{rgb}{0.85000,0.32500,0.09800}%
\definecolor{mycolor3}{rgb}{0.92900,0.69400,0.12500}%
\definecolor{mycolor4}{rgb}{0.49400,0.18400,0.55600}%
\begin{tikzpicture}

\begin{axis}[%
width=.35\linewidth,
height=.3\linewidth,
at={(0.745in,0.481in)},
scale only axis,
xmin=0,
xmax=50,
ymin=0,
ymax=1,
xmajorgrids,
ymajorgrids,
ymode=log,
axis background/.style={fill=white},
axis x line*=bottom,
axis y line*=left,
xlabel style={font=\color{white!15!black}},
xlabel={Number worst vertices},
x label style={at={(axis description cs:0.5,0.05)}, anchor=north,font=\footnotesize},
x tick label style = {font = \footnotesize},
ylabel style={font=\footnotesize{\color{white!15!black}}},
y tick label style = {font = \footnotesize},
ylabel={Probability},
y label style={at={(axis description cs:0.15,.5)},rotate=0,anchor=south},
legend style={at={(1,1.4)}, legend cell align=left, align=left, draw=white!15!black}
]
\addplot[only marks, mark=*, mark options={fill=mycolor1}, fill, mark size=3.0pt, draw=mycolor1, error bars/.cd, 
    y fixed,
    y dir=both, 
    y explicit,error bar style={line width=1.2pt,solid,mycolor1}]
  table[x=x, y=y,y error=error, row sep=crcr, row sep=crcr]{%
x	y       error \\
4	0.8005  0.235647087824144\\
8	0.468   0.251415989944952\\
12	0.2276  0.198280205769512\\
16	0.1431  0.198280205769512\\
20	0.0531  0.048585903305383\\
24	0.0328  0.027628970302927\\
28	0.0139  0.017871485668517\\
32	0.0159  0.031796068939415\\
36	0.0127  0.03349193932874\\
40	0.0032  0.0044\\
};
\addlegendentry{\footnotesize 500 Anneals}


\addplot[only marks, mark=*, mark options={fill=mycolor3}, fill, mark size=3.0pt, draw=mycolor3, error bars/.cd, 
    y fixed,
    y dir=both, 
    y explicit,error bar style={line width=1.2pt,solid,mycolor3}]
  table[x=x, y=y,y error=error, row sep=crcr, row sep=crcr]{%
x	y       error \\
40	0.0047  0.006814249775287\\
50	0.00074 0.001436802004453\\
};
\addlegendentry{\footnotesize 5000 Anneals}


\end{axis}

\end{tikzpicture}%

%% file: FIGURES_SUPP/success2.tikz
%
%
\definecolor{mycolor1}{rgb}{0.00000,0.44700,0.74100}%
\definecolor{mycolor2}{rgb}{0.85000,0.32500,0.09800}%
\definecolor{mycolor3}{rgb}{0.92900,0.69400,0.12500}%
\definecolor{mycolor4}{rgb}{0.49400,0.18400,0.55600}%
\begin{tikzpicture}

\begin{axis}[%
width=.35\linewidth,
height=.3\linewidth,
at={(0.745in,0.481in)},
scale only axis,
xmin=0,
xmax=51,
ymin=0,
ymax=1,
xmajorgrids,
ymajorgrids,
axis background/.style={fill=white},
axis x line*=bottom,
axis y line*=left,
xlabel style={font=\color{white!15!black}},
xlabel={Number worst vertices},
x label style={at={(axis description cs:0.5,0.05)}, anchor=north,font=\footnotesize},
x tick label style = {font=\footnotesize},
ylabel style={font=\footnotesize{\color{white!15!black}}},
ylabel={Probability},
y tick label style ={font=\footnotesize},
y label style={at={(axis description cs:0.25,.5)},rotate=0,anchor=south},
legend style={at={(1.2,1.4)}, legend cell align=left, align=left, draw=white!15!black}
]

\addplot[only marks, mark=diamond*, mark options={fill=mycolor2}, fill, mark size=3.0pt, draw=mycolor2, error bars/.cd, 
    y fixed,
    y dir=both, 
    y explicit,error bar style={line width=1.2pt,solid,mycolor2}]
  table[x=x, y=y,y error=error, row sep=crcr, row sep=crcr]{%
x	y       error\\
4	1       0\\
8	1       0\\
12	1       0\\
16	1       0\\
20	1       0\\
24	0.95    0.095518584579128\\
28	0.75    0.189776183964163\\
32	0.75    0.189776183964163\\
36	0.7     0.200840235012808\\
40	0.55    0.218036235520612\\
};
\addlegendentry{\footnotesize 500 Anneals}


\addplot[only marks, mark=diamond*, mark options={fill=mycolor4}, fill, mark size=3.0pt, draw=mycolor4, error bars/.cd, 
    y fixed,
    y dir=both, 
    y explicit,error bar style={line width=1.2pt,solid,mycolor4}]
  table[x=x, y=y,y error=error, row sep=crcr, row sep=crcr]{%
x	y       error\\
40	0.9     0.131480797076988 \\
50	0.45    0.218036235520612\\
};
\addlegendentry{\footnotesize 5000 Anneals}

\end{axis}

\end{tikzpicture}%

%% file: FIGURES_SUPP/chainlength.tikz
%
%
\definecolor{mycolor1}{rgb}{0.00000,0.44700,0.74100}%
\begin{tikzpicture}

\begin{axis}[%
width=.3\linewidth,
height=.3\linewidth,
at={(0.745in,0.481in)},
scale only axis,
xmin=0,
xmax=50,
ymin=0,
ymax=4,
xmajorgrids,
ymajorgrids,
axis background/.style={fill=white},
axis x line*=bottom,
axis y line*=left,
xticklabels={},
extra x ticks={0,10,20,30,40,50},
x tick label style = {},
xlabel style={font=\color{white!15!black}},
xlabel={\# worst vertices},
x label style={at={(axis description cs:0.5,0.05)}, anchor=north},
ylabel style={font=\color{white!15!black}},
ylabel={\small Avg. Maximum Chain Length},
y label style={at={(axis description cs:0.3,.5)},rotate=0,anchor=south},
legend style={at={(1,1.1)}, legend cell align=left, align=left, draw=white!15!black}
]
\addplot[only marks, mark=*, mark options={color=mycolor1}, fill, mark size=2.000pt, draw=mycolor1] table[row sep=crcr]{%
x	y\\
4	1\\
8	1\\
12	1.25\\
16	1.5\\
20	2.25\\
24	2.25\\
28	2.5\\
32	2.5\\
36	3.25\\
40	3.25\\
50	3.5\\
};

\end{axis}

\end{tikzpicture}%

%% file: FIGURES_SUPP/qubits.tikz
%
%
\definecolor{mycolor1}{rgb}{0.00000,0.44700,0.74100}%
\definecolor{mycolor2}{rgb}{0.85000,0.32500,0.09800}%
\definecolor{mycolor3}{rgb}{0.92900,0.69400,0.12500}%
\definecolor{mycolor4}{rgb}{0.49400,0.18400,0.55600}%
\begin{tikzpicture}

\begin{axis}[%
width=.3\linewidth,
height=.3\linewidth,
at={(0.745in,0.481in)},
scale only axis,
xmin=0,
xmax=50,
ymin=0,
ymax=100,
xmajorgrids,
ymajorgrids,
axis background/.style={fill=white},
axis x line*=bottom,
axis y line*=left,
xlabel style={font=\color{white!15!black}},
xlabel={\# worst vertices},
xticklabels={},
extra x ticks={0,10,20,30,40,50},
x tick label style = {},
x label style={at={(axis description cs:0.5,0.05)}, anchor=north},
ylabel style={font=\color{white!15!black}},
ylabel={Average \# of qubits},
y label style={at={(axis description cs:0.3,.5)},rotate=0,anchor=south},
legend style={at={(0.8,1.1)}, legend cell align=left, align=left, draw=white!15!black}
]
\addplot [color=mycolor2, line width = 1.75pt]
  table[row sep=crcr]{%
4	2\\
50	25\\
};
\addlegendentry{\footnotesize logical}
\addplot [color=mycolor1, line width = 1.75pt]
  table[row sep=crcr]{%
4	2\\
8	4\\
12	8\\
16	12\\
20	18.25\\
24	24\\
28	31.75\\
32	39.5\\
36	49.75\\
40	58.5\\
50	89.75\\
};
\addlegendentry{\footnotesize physical}

\end{axis}

\end{tikzpicture}%

%% file: FIGURES/timeSim.tex
%
%
\definecolor{mycolor1}{rgb}{0.00000,0.44700,0.74100}%
\begin{tikzpicture}

\begin{axis}[%
width=.4\linewidth,
height=.26\linewidth,
at={(0.5in,0.1in)},
scale only axis,
xmin=5,
xmax=40,
xlabel style={font=\color{white!15!black}},
xlabel={\# worst vertices},
ymin=-0.005,
ymax=0.03,
ylabel style={font=\color{white!15!black}},
ylabel={Processing time in s},
axis background/.style={fill=white},
xmajorgrids,
ymajorgrids
]
\addplot [color=mycolor1, only marks, mark=*, mark options={solid, mycolor1}, forget plot]
 plot [error bars/.cd, y dir=both, y explicit, error bar style={line width=0.5pt}, error mark options={line width=0.5pt, mark size=6.0pt, rotate=90}]
 table[row sep=crcr, y error plus index=2, y error minus index=3]{%
8	0.001339285	0.001560727679651	0.001560727679651\\
11	0.004375	0.005929270612816	0.005929270612816\\
14	0.004326921	0.004209589497525	0.004209589497525\\
17	0.003710939	0.003469511240296	0.003469511240296\\
20	0.00550987	0.004904426188384	0.004904426188384\\
23	0.007954535	0.003740214428604	0.003740214428604\\
26	0.00975	0.003987828704668	0.003987828704668\\
29	0.010044636	0.007309098793304	0.007309098793304\\
32	0.0141633	0.002914422077645	0.002914422077645\\
35	0.015165444	0.005199324931632	0.005199324931632\\
38	0.021114866	0.006481331652678	0.006481331652678\\
};
\end{axis}
\end{tikzpicture}%

%% file: FIGURES/time.tex
%
%
\definecolor{mycolor1}{rgb}{0.00000,0.44700,0.74100}%

\begin{tikzpicture}

\begin{axis}[%
width=.38\linewidth,
height=.3\linewidth,
at={(0.5in,0.1in)},
scale only axis,
xmin=5,
xmax=40,
xlabel style={font=\color{white!15!black}},
xlabel={\# worst vertices},
ymin=0,
ymax=0.5,
ylabel style={font=\color{white!15!black}},
ylabel={Processing time in s},
axis background/.style={fill=white},
xmajorgrids,
ymajorgrids
]
\addplot [color=mycolor1, only marks, mark=*, mark options={solid, mycolor1}, forget plot]
 plot [error bars/.cd, y dir=both, y explicit, error bar style={line width=0.5pt}, error mark options={line width=0.5pt, mark size=6.0pt, rotate=90}]
 table[row sep=crcr, y error plus index=2, y error minus index=3]{%
8	0.01875	0.006987712429687	0.006987712429687\\
11	0.03125	0	0\\
14	0.03125	0	0\\
17	0.046875	0.01104854345604	0.01104854345604\\
20	0.075	0.013072812914595	0.013072812914595\\
23	0.08125	0.017116329922037	0.017116329922037\\
26	0.121875	0.013072812914595	0.013072812914595\\
29	0.18125	0.02614562582919	0.02614562582919\\
32	0.25	0.024705294220066	0.024705294220066\\
35	0.309375	0.092702481088696	0.092702481088696\\
38	0.425	0.065736512399883	0.065736512399883\\
};
\end{axis}
\end{tikzpicture}%

%% file: FIGURES/faust_error2.tikz
%
%
\definecolor{mycolor1}{rgb}{0.00000,0.44700,0.74100}%
\definecolor{mycolor2}{rgb}{0.85000,0.32500,0.09800}%
\definecolor{mycolor3}{rgb}{0.92900,0.69400,0.12500}%
\definecolor{mycolor4}{rgb}{0.49400,0.18400,0.55600}%
\begin{tikzpicture}

\begin{axis}[%
width=.5\linewidth,
height=.35\linewidth,
at={(0.797in,0.617in)},
scale only axis,
xmin=0,
xmax=0.15,
ymin=30,
ymax=100,
xmajorgrids,
ymajorgrids,
axis background/.style={fill=white},
xlabel style={font=\color{white!15!black}},
xlabel={Geodesic error},
x label style = {font=\footnotesize},
x tick label style = {font=\footnotesize},
x label style={at={(axis description cs:0.5,0.05)}, anchor=north},
y label style = {font=\footnotesize},
ylabel style={font=\color{white!15!black}},
y label style = {font=\footnotesize},
ylabel={\% Correspondences},
y label style={at={(axis description cs:0.19,.5)},rotate=0,anchor=south},
title style={font=\bfseries},
legend style={at={(1.7,0.0)}, anchor=south east, legend cell align=left, align=left, draw=white!15!black, font=\footnotesize}
]

\addplot [color=red,line width=3.0pt]
  table[row sep=crcr]{%
0	52.6693227091633\\
0.01	53.3067729083665\\
0.02	59.1235059760956\\
0.03	68.00796812749\\
0.04	77.1713147410359\\
0.05	83.1872509960159\\
0.06	86.9322709163347\\
0.07	90.2390438247012\\
0.08	92.9880478087649\\
0.09	94.7410358565737\\
0.1	96.3346613545817\\
0.11	97.0517928286853\\
0.12	97.9282868525896\\
0.13	98.4860557768924\\
0.14	98.7250996015936\\
0.15	98.8446215139442\\
0.16	99.0438247011952\\
0.17	99.1633466135458\\
0.18	99.2430278884462\\
0.19	99.2828685258964\\
0.2	99.3227091633466\\
0.21	99.4422310756972\\
0.22	99.601593625498\\
0.23	99.8406374501992\\
0.24	99.9601593625498\\
0.25	99.9601593625498\\
0.26	99.9601593625498\\
0.27	99.9601593625498\\
0.28	99.9601593625498\\
0.29	99.9601593625498\\
0.3	99.9601593625498\\
0.31	99.9601593625498\\
0.32	99.9601593625498\\
0.33	99.9601593625498\\
0.34	99.9601593625498\\
0.35	99.9601593625498\\
0.36	99.9601593625498\\
0.37	99.9601593625498\\
0.38	99.9601593625498\\
0.39	99.9601593625498\\
0.4	99.9601593625498\\
0.41	99.9601593625498\\
0.42	99.9601593625498\\
0.43	99.9601593625498\\
0.44	99.9601593625498\\
0.45	99.9601593625498\\
0.46	99.9601593625498\\
0.47	99.9601593625498\\
0.48	99.9601593625498\\
0.49	99.9601593625498\\
0.5	99.9601593625498\\
0.51	99.9601593625498\\
0.52	99.9601593625498\\
0.53	99.9601593625498\\
0.54	99.9601593625498\\
0.55	99.9601593625498\\
0.56	99.9601593625498\\
0.57	99.9601593625498\\
0.58	99.9601593625498\\
0.59	99.9601593625498\\
0.6	99.9601593625498\\
0.61	99.9601593625498\\
0.62	99.9601593625498\\
0.63	99.9601593625498\\
0.64	99.9601593625498\\
0.65	100\\
0.66	100\\
0.67	100\\
0.68	100\\
0.69	100\\
0.7	100\\
0.71	100\\
0.72	100\\
0.73	100\\
0.74	100\\
0.75	100\\
0.76	100\\
0.77	100\\
0.78	100\\
0.79	100\\
0.8	100\\
0.81	100\\
0.82	100\\
0.83	100\\
0.84	100\\
0.85	100\\
0.86	100\\
0.87	100\\
0.88	100\\
0.89	100\\
0.9	100\\
0.91	100\\
0.92	100\\
0.93	100\\
0.94	100\\
0.95	100\\
0.96	100\\
0.97	100\\
0.98	100\\
0.99	100\\
1	100\\
};
\addlegendentry{Q-Match}

\addplot [color=mycolor2, dashed, line width=2.0pt]
  table[row sep=crcr]{%
0	31\\
0.01	31\\
0.02	33\\
0.03	35.6\\
0.04	41.6\\
0.05	45\\
0.06	50\\
0.07	54.6\\
0.08	59.6\\
0.09	64.8\\
0.1	67.6\\
0.11	71\\
0.12	75.4\\
0.13	79.6\\
0.14	81.8\\
0.15	82.4\\
0.16	84.8\\
0.17	85.6\\
0.18	87\\
0.19	88\\
0.2	89.2\\
0.21	89.8\\
0.22	90.4\\
0.23	91\\
0.24	91.4\\
0.25	91.8\\
0.26	92.2\\
0.27	92.6\\
0.28	92.8\\
0.29	93.2\\
0.3	93.2\\
0.31	93.6\\
0.32	94\\
0.33	94\\
0.34	94.6\\
0.35	94.8\\
0.36	95.4\\
0.37	95.8\\
0.38	95.8\\
0.39	96\\
0.4	96\\
0.41	96\\
0.42	96\\
0.43	96.4\\
0.44	97\\
0.45	97\\
0.46	97\\
0.47	97.2\\
0.48	97.2\\
0.49	97.2\\
0.5	97.2\\
};
\addlegendentry{FMs \cite{ovsjanikov2012functional}}

\addplot [color=mycolor1,dashed,line width=3.0pt]
  table[row sep=crcr]{%
0	28.1841638445412\\
0	54.6816479400749\\
0.01	56.3295880149813\\
0.02	64.4943820224719\\
0.03	72.5842696629214\\
0.04	79.625468164794\\
0.05	83.8202247191011\\
0.06	87.4906367041199\\
0.07	89.5880149812734\\
0.08	91.310861423221\\
0.09	93.1835205992509\\
0.1	94.1573033707865\\
0.11	94.9063670411985\\
0.12	95.6554307116105\\
0.13	96.3295880149813\\
0.14	96.8539325842697\\
0.15	97.5280898876404\\
0.16	97.9026217228464\\
0.17	98.2022471910112\\
0.18	98.3520599250936\\
0.19	98.8014981273408\\
0.2	99.1760299625468\\
0.21	99.5505617977528\\
0.22	99.8501872659176\\
0.23	100\\
0.24	100\\
0.25	100\\
0.26	100\\
0.27	100\\
0.28	100\\
0.29	100\\
0.3	100\\
0.31	100\\
0.32	100\\
0.33	100\\
0.34	100\\
0.35	100\\
0.36	100\\
0.37	100\\
0.38	100\\
0.39	100\\
0.4	100\\
0.41	100\\
0.42	100\\
0.43	100\\
0.44	100\\
0.45	100\\
0.46	100\\
0.47	100\\
0.48	100\\
0.49	100\\
0.5	100\\
};
\addlegendentry{SA \cite{holzschuh20simanneal}}

\addplot [color=mycolor4,line width=2.0pt,dashed]
  table[row sep=crcr]{%
0	87.5166002656043\\
0.01	88.3798140770252\\
0.02	91.0358565737052\\
0.03	93.0942895086321\\
0.04	95.0199203187251\\
0.05	96.2151394422311\\
0.06	96.7463479415671\\
0.07	97.0119521912351\\
0.08	97.4103585657371\\
0.09	98.406374501992\\
0.1	98.738379814077\\
0.11	98.937583001328\\
0.12	99.203187250996\\
0.13	99.33598937583\\
0.14	99.535192563081\\
0.15	99.667994687915\\
0.16	99.800796812749\\
0.17	99.867197875166\\
0.18	99.867197875166\\
0.19	99.867197875166\\
0.2	99.867197875166\\
0.21	99.867197875166\\
0.22	99.933598937583\\
0.23	99.933598937583\\
0.24	99.933598937583\\
0.25	99.933598937583\\
0.26	99.933598937583\\
0.27	100\\
0.28	100\\
0.29	100\\
0.3	100\\
0.31	100\\
0.32	100\\
0.33	100\\
0.34	100\\
0.35	100\\
0.36	100\\
0.37	100\\
0.38	100\\
0.39	100\\
0.4	100\\
0.41	100\\
0.42	100\\
0.43	100\\
0.44	100\\
0.45	100\\
0.46	100\\
0.47	100\\
0.48	100\\
0.49	100\\
0.5	100\\
};
\addlegendentry{ZO \cite{melzi19zoomout}}

\addplot [color=mycolor3,line width=3.0pt,dashed]
  table[row sep=crcr]{%
0	54.3824701195219\\
0.01	55.1792828685259\\
0.02	61.195219123506\\
0.03	69.6812749003984\\
0.04	79.1235059760956\\
0.05	85.0199203187251\\
0.06	88.2470119521912\\
0.07	91.5537848605578\\
0.08	94.5019920318725\\
0.09	95.6972111553785\\
0.1	96.9322709163347\\
0.11	97.9681274900398\\
0.12	98.605577689243\\
0.13	98.8446215139442\\
0.14	99.0438247011952\\
0.15	99.203187250996\\
0.16	99.2430278884462\\
0.17	99.4422310756972\\
0.18	99.5219123505976\\
0.19	99.5219123505976\\
0.2	99.5617529880478\\
0.21	99.6414342629482\\
0.22	99.6812749003984\\
0.23	99.8804780876494\\
0.24	99.9203187250996\\
0.25	99.9203187250996\\
0.26	99.9203187250996\\
0.27	99.9203187250996\\
0.28	99.9203187250996\\
0.29	99.9203187250996\\
0.3	99.9203187250996\\
0.31	99.9203187250996\\
0.32	99.9203187250996\\
0.33	99.9203187250996\\
0.34	99.9203187250996\\
0.35	99.9203187250996\\
0.36	99.9601593625498\\
0.37	99.9601593625498\\
0.38	99.9601593625498\\
0.39	99.9601593625498\\
0.4	99.9601593625498\\
0.41	99.9601593625498\\
0.42	99.9601593625498\\
0.43	99.9601593625498\\
0.44	99.9601593625498\\
0.45	99.9601593625498\\
0.46	100\\
0.47	100\\
0.48	100\\
0.49	100\\
0.5	100\\
0.51	100\\
0.52	100\\
0.53	100\\
0.54	100\\
0.55	100\\
0.56	100\\
0.57	100\\
0.58	100\\
0.59	100\\
0.6	100\\
0.61	100\\
0.62	100\\
0.63	100\\
0.64	100\\
0.65	100\\
0.66	100\\
0.67	100\\
0.68	100\\
0.69	100\\
0.7	100\\
0.71	100\\
0.72	100\\
0.73	100\\
0.74	100\\
0.75	100\\
0.76	100\\
0.77	100\\
0.78	100\\
0.79	100\\
0.8	100\\
0.81	100\\
0.82	100\\
0.83	100\\
0.84	100\\
0.85	100\\
0.86	100\\
0.87	100\\
0.88	100\\
0.89	100\\
0.9	100\\
0.91	100\\
0.92	100\\
0.93	100\\
0.94	100\\
0.95	100\\
0.96	100\\
0.97	100\\
0.98	100\\
0.99	100\\
1	100\\
};
\addlegendentry{SA-Match-5}
\end{axis}

\end{tikzpicture}%

%% file: FIGURES/worst2.tikz
%
%
\definecolor{mycolor1}{rgb}{0.00000,0.44700,0.74100}%
\definecolor{mycolor2}{rgb}{0.85000,0.32500,0.09800}%
\definecolor{mycolor3}{rgb}{0.92900,0.69400,0.12500}%
\definecolor{mycolor4}{rgb}{0.49400,0.18400,0.55600}%
\begin{tikzpicture}

\begin{axis}[%
width=.5\linewidth,
height=.35\linewidth,
at={(0.5in,0.1in)},
scale only axis,
xmin=1,
xmax=15,
ymin=0,
ymax=30000,
xmajorgrids,
ymajorgrids,
axis background/.style={fill=white},
xlabel style={font=\color{white!15!black}},
xlabel={\small \# Iteration},
x tick label style = {font=\footnotesize},
y tick label style = {font =\footnotesize},
x label style={at={(axis description cs:0.5,0.1)}, anchor=north},
ylabel style={font=\color{white!15!black}},
ylabel={\small Energy \eqref{eq:DASProblem}},
y label style={at={(axis description cs:0.2,.5)},rotate=0,anchor=south},
title style={font=\bfseries},
legend style={at={(1.5,0.48)}, anchor=south east, legend cell align=left, align=left, draw=white!15!black, font=\footnotesize}
]
\addplot [color=mycolor1, line width=2pt]
  table[row sep=crcr]{%
1	28967.7740945007\\
2	20462.8711526334\\
3	17938.8933045271\\
4	15133.3969379716\\
5	12933.9466254587\\
6	11590.26628047\\
7	10428.6223381632\\
8	9700.66895599068\\
9	9200.43498568485\\
10	8782.83810686946\\
11	8478.04747317207\\
12	8149.97998745927\\
13	7964.23038636745\\
14	7853.93685846935\\
15	7741.44210339278\\
16	7585.12901476899\\
17	7506.80118382658\\
18	7506.80118382658\\
};
\addlegendentry{Worst 26}

\addplot [color=mycolor3, line width=2.0pt]
  table[row sep=crcr]{%
1	28967.7740945007\\
2	19785.5638124038\\
3	15965.625080292\\
4	12559.5330345494\\
5	10716.6133549037\\
6	9161.67187490522\\
7	8331.93259485467\\
8	7855.8441526553\\
9	7422.46809373534\\
10	7005.97813927963\\
11	6789.73765964623\\
12	6755.55681831424\\
13	6758.7013013385\\
14	6698.82219891283\\
15	6605.60138840007\\
16	6566.61404866082\\
17	6562.50554973589\\
18	6501.10204693944\\
19	6525.49124386912\\
};
\addlegendentry{Worst 40}

\addplot [color=mycolor4, line width=2.0pt]
  table[row sep=crcr]{%
1	28967.7740945007\\
2	18455.8661354276\\
3	13636.3958207801\\
4	11043.0417990707\\
5	9409.15456174472\\
6	8950.78079372073\\
7	8002.40795811593\\
8	7449.97648451161\\
9	6921.51806768985\\
10	6651.48589420183\\
11	6574.70936789903\\
12	6441.67271594747\\
13	6285.61209839924\\
14	6082.44440415722\\
15	6166.50557865852\\
};
\addlegendentry{Worst 50}

\addplot [color=gray, dashed, line width = 1.5pt]
  table[row sep=crcr]{%
1	5479.47419680382\\
28	5479.47419680382\\
};

\end{axis}

\end{tikzpicture}%

%% file: FIGURES/qapbur.tex
%
%
\definecolor{mycolor1}{rgb}{0.00000,0.44700,0.74100}%
\begin{tikzpicture}

\begin{axis}[%
width=.35\linewidth,
height=.18\linewidth,
at={(0.758in,0.481in)},
scale only axis,
xmin=0.65,
xmax=8.35,
xlabel style={font=\color{white!15!black}},
xlabel={Instance in \cite{burkard77qaplib}},
ymin=0,
ymax=1.4,
ylabel style={font=\color{white!15!black}},
ylabel={Relative error in \%},
axis background/.style={fill=white},
xmajorgrids,
ymajorgrids
]
\addplot [color=red, only marks, mark size=2.0pt, mark=*, mark options={solid, red}, forget plot]
  table[row sep=crcr]{%
1	0.443863363720293\\
2	0.276411971967483\\
3	1.07678657476466\\
4	0.025253681738186\\
5	0.303682336284145\\
6	0.398620428530181\\
7	0.410203562813798\\
8	0.765046012922442\\
};
\addplot [color=blue, only marks, mark size=2.0pt, mark=*, mark options={solid, blue}, forget plot]
  table[row sep=crcr]{%
1	0.423519395872607\\
2	0.494099823670483\\
3	0.340477206159437\\
4	1.27817650099117\\
5	0.212627757185561\\
6	0.033976336605285\\
7	0.01867122551638\\
8	0.453156075415945\\
};
\end{axis}
\end{tikzpicture}%

%% file: FIGURES/qaprest.tex
%
%
\definecolor{mycolor1}{rgb}{0.00000,0.44700,0.74100}%
\begin{tikzpicture}

\begin{axis}[%
width=.35\linewidth,
height=.18\linewidth,
at={(0.758in,0.481in)},
scale only axis,
xmin=0.5,
xmax=11.5,
xlabel style={font=\color{white!15!black}},
xlabel={Instance in \cite{scrqaplib,hadqaplib,rouqaplib}},
ymin=0,
ymax=15,
axis background/.style={fill=white},
xmajorgrids,
ymajorgrids
]
\addplot [color=red, only marks, mark size=2.0pt, mark=*, mark options={solid, red}, forget plot]
  table[row sep=crcr]{%
1	7.20152817574021\\
2	14.5991396167384\\
3	5.26401890393529\\
4	2.05811138014528\\
5	0.220264317180617\\
6	0.376344086021505\\
7	0.335946248600224\\
8	2.10921698930945\\
9	5.71227200163038\\
10	8.51020581011264\\
11	4.85388451349511\\
};
\addplot [color=blue, only marks, mark size=2.0pt, mark=*, mark options={solid, blue}, forget plot]
  table[row sep=crcr]{%
1	12.8748806112703\\
2	14.0398904966758\\
3	3.90075433972553\\
4	0\\
5	0.881057268722467\\
6	0.806451612903226\\
7	1.04516610675625\\
8	0.722334585379948\\
9	6.93930233348052\\
10	5.36630812230033\\
11	3.99491676337864\\
};
\end{axis}
\end{tikzpicture}%

%% file: FIGURES/qapnug.tex
%
%
\begin{tikzpicture}

\begin{axis}[%
width=.35\linewidth,
height=.18\linewidth,
at={(0.758in,0.481in)},
scale only axis,
xmin=0,
xmax=14.7,
xlabel style={font=\color{white!15!black}},
xlabel={Instance in \cite{nugqaplib}},
ymin=0,
ymax=9,
ylabel style={font=\color{white!15!black}},
ylabel={Relative error in \%},
axis background/.style={fill=white},
xmajorgrids,
ymajorgrids
]
\addplot [color=red, only marks, mark size=2.0pt, mark=*, mark options={solid, red}, forget plot]
  table[row sep=crcr]{%
1	6.9204152249135\\
2	1.18343195266272\\
3	2.48447204968944\\
4	4.51612903225806\\
5	8.66050808314088\\
6	0.310880829015544\\
7	1.40077821011673\\
8	5.57834290401969\\
9	3.2258064516129\\
10	4.12844036697248\\
11	6.94444444444444\\
12	6.03744745892243\\
13	3.5230352303523\\
14	3.72305682560418\\
};
\addplot [color=blue, only marks, mark size=2.0pt, mark=*, mark options={solid, blue}, forget plot]
  table[row sep=crcr]{%
1	2.7681660899654\\
2	6.11439842209073\\
3	5.21739130434783\\
4	6.61290322580645\\
5	4.04157043879908\\
6	2.38341968911917\\
7	4.3579766536965\\
8	3.19934372436423\\
9	3.28142380422692\\
10	4.07110091743119\\
11	5.23504273504274\\
12	5.23500191058464\\
13	6.73635307781649\\
14	5.81319399085565\\
};
\end{axis}
\end{tikzpicture}%

%% file: FIGURES/qapesc.tex
%
%
%
\definecolor{mycolor1}{rgb}{0.00000,0.44700,0.74100}%
\begin{tikzpicture}

\begin{axis}[%
width=.35\linewidth,
height=.18\linewidth,
at={(0.758in,0.481in)},
scale only axis,
xmin=0.55,
xmax=10.45,
xlabel style={font=\color{white!15!black}},
xlabel={Instance in \cite{esc16qaplib}},
ymin=0,
ymax=25,
axis background/.style={fill=white},
xmajorgrids,
ymajorgrids,
legend style={at={(0.5,0.97)}, anchor=north, legend cell align=left, align=left, draw=white!15!black}
]
\addplot [color=red, only marks, mark size=2.0pt, mark=*, mark options={solid, red}]
  table[row sep=crcr]{%
1	2.94117647058823\\
2	0\\
3	0\\
4	0\\
5	0\\
6	0\\
7	0\\
8	0\\
9	0\\
10	0\\
};
\addlegendentry{AQC}

\addplot [color=blue, only marks, mark size=2.0pt, mark=*, mark options={solid, blue}]
  table[row sep=crcr]{%
1	0\\
2	0\\
3	0\\
4	0\\
5	0\\
6	0\\
7	0\\
8	0\\
9	0\\
10	25\\
};
\addlegendentry{SA}

\end{axis}
\end{tikzpicture}%

%% file: FIGURES/faust_error_badinit.tikz
%
%
\definecolor{mycolor1}{rgb}{0.00000,0.44700,0.74100}%
\definecolor{mycolor2}{rgb}{0.85000,0.32500,0.09800}%
\definecolor{mycolor3}{rgb}{0.92900,0.69400,0.12500}%
\definecolor{mycolor4}{rgb}{0.49400,0.18400,0.55600}%
\begin{tikzpicture}

\begin{axis}[%
width=.5\linewidth,
height=.35\linewidth,
at={(0.797in,0.617in)},
scale only axis,
xmin=0,
xmax=0.2,
ymin=0,
ymax=100,
xmajorgrids,
ymajorgrids,
axis background/.style={fill=white},
xlabel style={font=\color{white!15!black}},
xlabel={Geodesic error},
x label style = {font=\footnotesize},
x tick label style = {font=\footnotesize},
x label style={at={(axis description cs:0.5,0.05)}, anchor=north},
y label style = {font=\footnotesize},
ylabel style={font=\color{white!15!black}},
y label style = {font=\footnotesize},
ylabel={\% Correspondences},
y label style={at={(axis description cs:0.19,.5)},rotate=0,anchor=south},
title style={font=\bfseries},
legend style={at={(1.8,0.0)}, anchor=south east, legend cell align=left, align=left, draw=white!15!black, font=\footnotesize}
]
\addplot [color=mycolor2,line width=1.5pt]
  table[row sep=crcr]{%
0	31\\
0.01	31\\
0.02	33\\
0.03	35.6\\
0.04	41.6\\
0.05	45\\
0.06	50\\
0.07	54.6\\
0.08	59.6\\
0.09	64.8\\
0.1	67.6\\
0.11	71\\
0.12	75.4\\
0.13	79.6\\
0.14	81.8\\
0.15	82.4\\
0.16	84.8\\
0.17	85.6\\
0.18	87\\
0.19	88\\
0.2	89.2\\
0.21	89.8\\
0.22	90.4\\
0.23	91\\
0.24	91.4\\
0.25	91.8\\
0.26	92.2\\
0.27	92.6\\
0.28	92.8\\
0.29	93.2\\
0.3	93.2\\
0.31	93.6\\
0.32	94\\
0.33	94\\
0.34	94.6\\
0.35	94.8\\
0.36	95.4\\
0.37	95.8\\
0.38	95.8\\
0.39	96\\
0.4	96\\
0.41	96\\
0.42	96\\
0.43	96.4\\
0.44	97\\
0.45	97\\
0.46	97\\
0.47	97.2\\
0.48	97.2\\
0.49	97.2\\
0.5	97.2\\
};
\addlegendentry{FMs [50]}

\addplot [color=mycolor1,line width=1.5pt]
  table[row sep=crcr]{%
0	28.1841638445412\\
0	54.6816479400749\\
0.01	56.3295880149813\\
0.02	64.4943820224719\\
0.03	72.5842696629214\\
0.04	79.625468164794\\
0.05	83.8202247191011\\
0.06	87.4906367041199\\
0.07	89.5880149812734\\
0.08	91.310861423221\\
0.09	93.1835205992509\\
0.1	94.1573033707865\\
0.11	94.9063670411985\\
0.12	95.6554307116105\\
0.13	96.3295880149813\\
0.14	96.8539325842697\\
0.15	97.5280898876404\\
0.16	97.9026217228464\\
0.17	98.2022471910112\\
0.18	98.3520599250936\\
0.19	98.8014981273408\\
0.2	99.1760299625468\\
0.21	99.5505617977528\\
0.22	99.8501872659176\\
0.23	100\\
0.24	100\\
0.25	100\\
0.26	100\\
0.27	100\\
0.28	100\\
0.29	100\\
0.3	100\\
0.31	100\\
0.32	100\\
0.33	100\\
0.34	100\\
0.35	100\\
0.36	100\\
0.37	100\\
0.38	100\\
0.39	100\\
0.4	100\\
0.41	100\\
0.42	100\\
0.43	100\\
0.44	100\\
0.45	100\\
0.46	100\\
0.47	100\\
0.48	100\\
0.49	100\\
0.5	100\\
};
\addlegendentry{SA [27]}

\addplot [color=mycolor3,line width=1.5pt]
  table[row sep=crcr]{%
0	87.5166002656043\\
0.01	88.3798140770252\\
0.02	91.0358565737052\\
0.03	93.0942895086321\\
0.04	95.0199203187251\\
0.05	96.2151394422311\\
0.06	96.7463479415671\\
0.07	97.0119521912351\\
0.08	97.4103585657371\\
0.09	98.406374501992\\
0.1	98.738379814077\\
0.11	98.937583001328\\
0.12	99.203187250996\\
0.13	99.33598937583\\
0.14	99.535192563081\\
0.15	99.667994687915\\
0.16	99.800796812749\\
0.17	99.867197875166\\
0.18	99.867197875166\\
0.19	99.867197875166\\
0.2	99.867197875166\\
0.21	99.867197875166\\
0.22	99.933598937583\\
0.23	99.933598937583\\
0.24	99.933598937583\\
0.25	99.933598937583\\
0.26	99.933598937583\\
0.27	100\\
0.28	100\\
0.29	100\\
0.3	100\\
0.31	100\\
0.32	100\\
0.33	100\\
0.34	100\\
0.35	100\\
0.36	100\\
0.37	100\\
0.38	100\\
0.39	100\\
0.4	100\\
0.41	100\\
0.42	100\\
0.43	100\\
0.44	100\\
0.45	100\\
0.46	100\\
0.47	100\\
0.48	100\\
0.49	100\\
0.5	100\\
};
\addlegendentry{ZO [42]}

\addplot [color=red,line width=2.0pt]
  table[row sep=crcr]{%
0	52.6693227091633\\
0.01	53.3067729083665\\
0.02	59.1235059760956\\
0.03	68.00796812749\\
0.04	77.1713147410359\\
0.05	83.1872509960159\\
0.06	86.9322709163347\\
0.07	90.2390438247012\\
0.08	92.9880478087649\\
0.09	94.7410358565737\\
0.1	96.3346613545817\\
0.11	97.0517928286853\\
0.12	97.9282868525896\\
0.13	98.4860557768924\\
0.14	98.7250996015936\\
0.15	98.8446215139442\\
0.16	99.0438247011952\\
0.17	99.1633466135458\\
0.18	99.2430278884462\\
0.19	99.2828685258964\\
0.2	99.3227091633466\\
0.21	99.4422310756972\\
0.22	99.601593625498\\
0.23	99.8406374501992\\
0.24	99.9601593625498\\
0.25	99.9601593625498\\
0.26	99.9601593625498\\
0.27	99.9601593625498\\
0.28	99.9601593625498\\
0.29	99.9601593625498\\
0.3	99.9601593625498\\
0.31	99.9601593625498\\
0.32	99.9601593625498\\
0.33	99.9601593625498\\
0.34	99.9601593625498\\
0.35	99.9601593625498\\
0.36	99.9601593625498\\
0.37	99.9601593625498\\
0.38	99.9601593625498\\
0.39	99.9601593625498\\
0.4	99.9601593625498\\
0.41	99.9601593625498\\
0.42	99.9601593625498\\
0.43	99.9601593625498\\
0.44	99.9601593625498\\
0.45	99.9601593625498\\
0.46	99.9601593625498\\
0.47	99.9601593625498\\
0.48	99.9601593625498\\
0.49	99.9601593625498\\
0.5	99.9601593625498\\
0.51	99.9601593625498\\
0.52	99.9601593625498\\
0.53	99.9601593625498\\
0.54	99.9601593625498\\
0.55	99.9601593625498\\
0.56	99.9601593625498\\
0.57	99.9601593625498\\
0.58	99.9601593625498\\
0.59	99.9601593625498\\
0.6	99.9601593625498\\
0.61	99.9601593625498\\
0.62	99.9601593625498\\
0.63	99.9601593625498\\
0.64	99.9601593625498\\
0.65	100\\
0.66	100\\
0.67	100\\
0.68	100\\
0.69	100\\
0.7	100\\
0.71	100\\
0.72	100\\
0.73	100\\
0.74	100\\
0.75	100\\
0.76	100\\
0.77	100\\
0.78	100\\
0.79	100\\
0.8	100\\
0.81	100\\
0.82	100\\
0.83	100\\
0.84	100\\
0.85	100\\
0.86	100\\
0.87	100\\
0.88	100\\
0.89	100\\
0.9	100\\
0.91	100\\
0.92	100\\
0.93	100\\
0.94	100\\
0.95	100\\
0.96	100\\
0.97	100\\
0.98	100\\
0.99	100\\
1	100\\
};
\addlegendentry{Q-Match}

\addplot [color=red, dashed, line width=2.0pt]
  table[row sep=crcr]{%
0	18.6254980079681\\
0.01	19.4223107569721\\
0.02	28.6852589641434\\
0.03	39.9402390438247\\
0.04	56.4741035856574\\
0.05	67.6294820717131\\
0.06	78.4860557768924\\
0.07	85.1593625498008\\
0.08	89.8406374501992\\
0.09	93.4262948207171\\
0.1	94.7211155378486\\
0.11	95.9163346613546\\
0.12	97.5099601593625\\
0.13	98.207171314741\\
0.14	98.406374501992\\
0.15	98.605577689243\\
0.16	98.9043824701195\\
0.17	99.003984063745\\
0.18	99.1035856573705\\
0.19	99.3027888446215\\
0.2	99.402390438247\\
0.21	99.7011952191235\\
0.22	99.7011952191235\\
0.23	99.9003984063745\\
0.24	100\\
0.25	100\\
0.26	100\\
0.27	100\\
0.28	100\\
0.29	100\\
0.3	100\\
0.31	100\\
0.32	100\\
0.33	100\\
0.34	100\\
0.35	100\\
0.36	100\\
0.37	100\\
0.38	100\\
0.39	100\\
0.4	100\\
0.41	100\\
0.42	100\\
0.43	100\\
0.44	100\\
0.45	100\\
0.46	100\\
0.47	100\\
0.48	100\\
0.49	100\\
0.5	100\\
};
\addlegendentry{Q-Match (bad init)}

\end{axis}

\end{tikzpicture}%

%% file: FIGURES/worst_badinit.tikz
%
%
\definecolor{mycolor1}{rgb}{0.00000,0.44700,0.74100}%
\definecolor{mycolor2}{rgb}{0.85000,0.32500,0.09800}%
\definecolor{mycolor3}{rgb}{0.92900,0.69400,0.12500}%
\definecolor{mycolor4}{rgb}{0.49400,0.18400,0.55600}%
	\definecolor{ao}{rgb}{0.0, 0.5, 0.0}
\begin{tikzpicture}

\begin{axis}[%
width=.5\linewidth,
height=.35\linewidth,
at={(0.5in,0.1in)},
scale only axis,
xmin=1,
xmax=28,
ymin=0,
ymax=150000,
xmajorgrids,
ymajorgrids,
ymode = log,
axis background/.style={fill=white},
xlabel style={font=\color{white!15!black}},
xlabel={\small \# Iteration},
x tick label style = {font=\footnotesize},
y tick label style = {font =\footnotesize},
x label style={at={(axis description cs:0.5,0.1)}, anchor=north},
ylabel style={font=\color{white!15!black}},
ylabel={\small Energy (1)},
y label style={at={(axis description cs:0.2,.5)},rotate=0,anchor=south},
title style={font=\bfseries},
legend style={at={(1.4,0.48)}, anchor=south east, legend cell align=left, align=left, draw=white!15!black, font=\footnotesize}
]

\addplot [color=mycolor1, line width=2pt]
  table[row sep=crcr]{%
1	28967.7740945007\\
2	20462.8711526334\\
3	17938.8933045271\\
4	15133.3969379716\\
5	12933.9466254587\\
6	11590.26628047\\
7	10428.6223381632\\
8	9700.66895599068\\
9	9200.43498568485\\
10	8782.83810686946\\
11	8478.04747317207\\
12	8149.97998745927\\
13	7964.23038636745\\
14	7853.93685846935\\
15	7741.44210339278\\
16	7585.12901476899\\
17	7506.80118382658\\
18	7506.80118382658\\
};
\addlegendentry{26}

\addplot [color=mycolor3, line width=2.0pt]
  table[row sep=crcr]{%
1	28967.7740945007\\
2	19785.5638124038\\
3	15965.625080292\\
4	12559.5330345494\\
5	10716.6133549037\\
6	9161.67187490522\\
7	8331.93259485467\\
8	7855.8441526553\\
9	7422.46809373534\\
10	7005.97813927963\\
11	6789.73765964623\\
12	6755.55681831424\\
13	6758.7013013385\\
14	6698.82219891283\\
15	6605.60138840007\\
16	6566.61404866082\\
17	6562.50554973589\\
18	6501.10204693944\\
19	6525.49124386912\\
};
\addlegendentry{40}

\addplot [color=mycolor4, line width=2.0pt]
  table[row sep=crcr]{%
1	28967.7740945007\\
2	18455.8661354276\\
3	13636.3958207801\\
4	11043.0417990707\\
5	9409.15456174472\\
6	8950.78079372073\\
7	8002.40795811593\\
8	7449.97648451161\\
9	6921.51806768985\\
10	6651.48589420183\\
11	6574.70936789903\\
12	6441.67271594747\\
13	6285.61209839924\\
14	6082.44440415722\\
15	6166.50557865852\\
};
\addlegendentry{50}

\addplot [color=mycolor1, dashed, line width=2pt]
  table[row sep=crcr]{%
1   147433.96787143423 \\
2   142192.62323329243 \\
3   136253.8536518901 \\
4   130436.40988619089 \\
5   124005.34656277008 \\
6   116808.51441409238 \\
7   109733.12023647716 \\
8   102693.38723233916 \\
9   95659.89715728436 \\
10  88825.67772833773 \\
11  82307.68891305065 \\
12  76711.22083110493 \\
13  69893.72207025846 \\
14  62784.85400370065 \\
15  56765.36960602349 \\
16  51311.35174502351 \\
17  46837.5554608681 \\
18  43151.7755188139 \\
19  40256.72587303916 \\
20  37572.616244433884 \\
21  36131.906384588365 \\
22  34330.30496518413 \\
23  32837.31737364331 \\
24  30876.21452408147 \\
25  29697.72544178123 \\
26  27994.602959265187 \\
27  26337.31466217956 \\
28  24797.301165198758 \\
};

\addplot [color=mycolor3, dashed, line width=2.0pt]
  table[row sep=crcr]{%
1	147433.967871434\\
2	139285.259171723\\
3	130099.118486156\\
4	119377.254111645\\
5	108436.933593728\\
6	96981.8241760608\\
7	85920.8604686566\\
8	75982.2648127705\\
9	65536.803717107\\
10	55918.6261517324\\
11	47098.0830814349\\
12	39784.9621535889\\
13	34584.8688409008\\
14	30939.7178943774\\
15	28150.1913678377\\
16	25741.782177736\\
17	23904.3033815583\\
18	22061.6577600605\\
19	20470.3228265976\\
20	19243.2516025701\\
21	18247.2588834436\\
22	17567.0634846369\\
23	16873.0515923245\\
24	16218.3357092174\\
25	15874.6668653572\\
26	15454.7181204576\\
27	15159.0309093176\\
28	14969.6691337843\\
};

\addplot [color=mycolor4, dashed, line width=2.0pt]
  table[row sep=crcr]{%
1	147433.967871434\\
2	136338.351274613\\
3	123248.541818405\\
4	109704.278389896\\
5	96067.5600928574\\
6	82892.9291952542\\
7	68779.8087122125\\
8	55858.9776679827\\
9	44532.2421190974\\
10	37135.0179906324\\
11	32395.782979115\\
12	27833.6023746997\\
13	24561.8250611795\\
14	22182.2027324525\\
15	19950.554820416\\
16	18297.0265883002\\
17	17237.8173072727\\
18	16651.6463981042\\
19	16362.1798500775\\
20	15639.8211553784\\
21	15435.1109718776\\
22	15300.1793926107\\
23	15211.9103510424\\
24	14903.6054690431\\
25	14643.6355651243\\
26	14505.7394595855\\
27	14238.1870571144\\
28	14117.1063499428\\
};

\addplot [color=gray, dashed, line width = 1.5pt]
  table[row sep=crcr]{%
1	5479.47419680382\\
28	5479.47419680382\\
};

\end{axis}

\end{tikzpicture}%